\documentclass[runningheads]{llncs}

\usepackage{eccv/eccv}

\usepackage{eccv/eccvabbrv}

\usepackage{graphicx}
\usepackage{booktabs}

\usepackage[accsupp]{axessibility}  
\usepackage[numbers]{natbib}
\usepackage{subcaption}
\usepackage{mathtools}
\usepackage{algorithm}
\usepackage{algpseudocode}
\usepackage{algorithm}
\usepackage{algpseudocode}
\newtheorem{assumption}[theorem]{Assumption}
\usepackage{caption}
\usepackage{wrapfig}

\usepackage{hyperref}

\usepackage{orcidlink}

\begin{document}

\title{GeoEdit: Local Frames for Fast, Training-Free On-Manifold Editing in Diffusion Models} 

\titlerunning{GeoEdit: Local Frames for Fast, Training-Free On-Manifold Editing}

\author{
Yiming Zhang\inst{1} \and
Sitong Liu\inst{2} \and
Ke Li\inst{3} \and
Zhihong Wu\inst{3} \and
Alex Cloninger\inst{1} \and
Melvin Leok\inst{1}
}

\authorrunning{Y.~Zhang et al.}

\institute{
University of California San Diego, La Jolla, CA, USA
\and
University of Washington, Seattle, WA, USA
\and
Xidian University, Xi'an, China
}


\maketitle

\begin{abstract}
  Diffusion models are a leading paradigm for data generation, but training-free editing typically re-runs the full denoising trajectory for every edit strength, making iterative refinement expensive. To address this issue, we instead edit near the data manifold, where small local updates can replace repeated re-synthesis. To enable this, we estimate a local manifold tangent space directly from perturbed samples and prove that this sample-based estimator closely approximates the true tangent. Building on this guarantee, we devise a Jacobian-free algorithm that constructs a tangent frame via small perturbations to the initial noise and alternates small tangent moves with diffusion-based projections. Updates within this frame follow principled on-manifold directions while suppressing off-manifold drift, enabling fine-grained edits without full re-diffusion or additional training. Edit strength is controlled by the number of steps for rapid, continuous adjustments that preserve fidelity and plug into existing samplers. Empirically, the resulting tangent directions yield smooth, semantic unsupervised traversals and effective CLIP-guided optimization, demonstrating practical interactive continuous editing.
\end{abstract}

\section{Introduction}
\label{sec:intro}


Diffusion models have achieved remarkable progress in image synthesis~\cite{ho2020denoising,song2021score,dhariwal2021diffusion}, and large-scale variants now enable unprecedented creative control for everyday users~\cite{rombach2022high,podell2023sdxl}. This momentum has spurred intense interest in adapting such models for image editing. In particular, training-free editing has gained attention for enabling targeted modifications at inference without any retraining~\cite{meng2021sdedit,hertz2022prompt,couairon2022diffedit,mokady2023null}. This paradigm promises practical, scalable editing while preserving fidelity to the source content.

Existing training-free editing methods typically steer the diffusion trajectory by altering prompts~\cite{ parmar2023zero, tumanyan2023plug}, classifier-free guidance~\cite{mokady2023null}, or cross-attention maps~\cite{Cao2023MasaCtrl} along the entire sampling process (often after an inversion warm-start). As a result, the edit strength is entangled with the stochastic sampling path: changing it usually requires restarting from noise or re-diffusing from an early timestep with new guidance schedules. This coupling makes continuous, 
interactive adjustment inconvenient and latency-heavy, and it often yields non-monotonic responses to small strength changes, hindering fine control while risking unnecessary drift from the source image. 

Consequently, it is desirable to perform edits in the vicinity of the data manifold captured by the diffusion model~\cite{stanczuk2024diffusion} so that generation can proceed with small, local updates rather than full re-synthesis. Operating near this manifold enables reuse of cached noise and partial denoising from late time-steps instead of restarting from pure noise. Together, this approach decouples edit strength from full re-diffusion, markedly reducing latency and compute while preserving fidelity to the original content. Recent work controls edits by following tangent directions of the learned data manifold: (i) low-dimensional subspaces aligned with semantic controls \citep{chen2024exploring}, and (ii) a Riemannian view using tangent spaces, geodesics, and parallel transport \citep{park2023understanding}. In these works, the tangents are estimated via local linearization of the denoiser/score, which entails Jacobian or Jacobian–vector product evaluations. Computing these quantities is costly for modern U-Nets. In multi-step editing, error accumulates and quickly invalidates the local linearization, so the tangent must be recomputed frequently—often at every step. In LOCO Edit~\cite{chen2024exploring}, editing directions are obtained by locally linearizing the
posterior mean predictor \(\hat{\mathbf{x}}_{0}(t,\mathbf{e})\) of a pretrained diffusion model along the inverted trajectory.
Let \(J(t)=\partial \hat{\mathbf{x}}_{0}(t,\mathbf{e})/\partial \mathbf{e}=U(t)\Sigma(t)V(t)^{\top}\) be its Jacobian SVD.
The left singular vectors \(U(t)\) give principal image-space directions, while the right singular vectors \(V(t)\)
correspond to conditioning-space directions that act as estimated tangent vectors for on-manifold editing. However, using \(t\!\approx\!0\) to extract or apply these directions is suboptimal. At late stage the denoising dynamics are detail-dominated: \(\hat{\mathbf{x}}_{0}(t,\mathbf{e})\) primarily refines high-frequency textures, \begin{wrapfigure}{r}{0.48\columnwidth}
  \centering
  \vspace{-0.7\baselineskip} 
  \includegraphics[width=\linewidth]{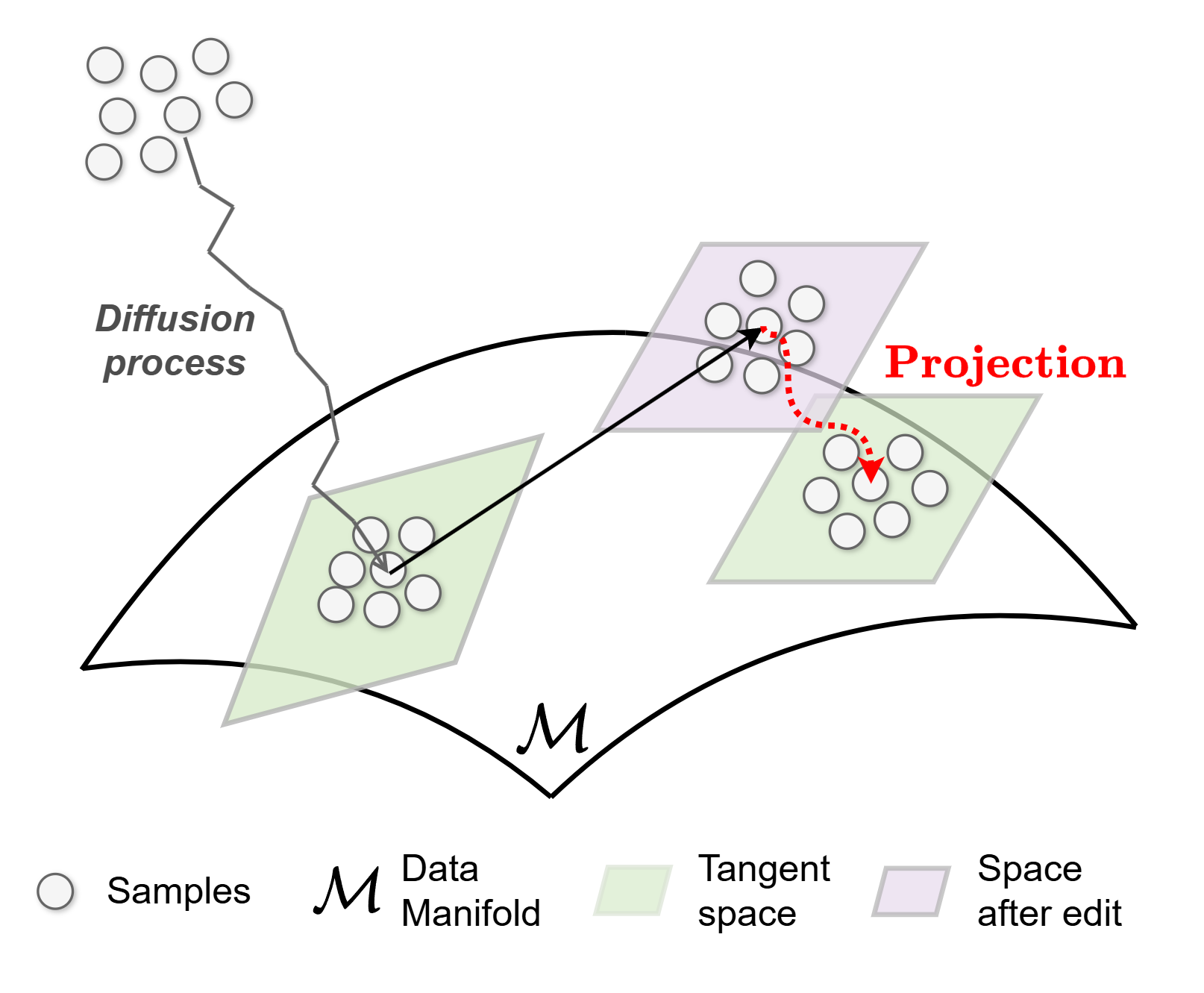}
  \caption{The illustration of proposed algorithm. Perturbed samples are first flowed onto the data manifold, yielding an initial tangent estimate. During editing, we alternate small tangent moves with diffusion-based projections, which refine the tangent space adaptively.}
  \label{fig:algorithm}
  \vspace{-2.5\baselineskip} 
\end{wrapfigure}so \(J(t)\) spreads sensitivity over many fine-grained modes, weakening any low-rank semantic structure as exhibited in \cite{chen2024exploring}.

To address these issues, we first provide a theoretical guarantee showing that the subspace spanned by $k$ ambient perturbations converges to the true tangent space as the tube radius and curvature shrink (Theorem~\ref{thm:tangent-approx}). This validates the feasibility of estimating reliable tangent directions from sample statistics. Leveraging this result, we introduce a manifold-aware, Jacobian-free
editing algorithm that integrates seamlessly with guided diffusion samplers. It approximates on-manifold tangents via small perturbations of the initial noise, rather than through explicit Jacobian evaluations, and is most reliable in the mid-to-late denoising regime where data mass concentrates near the manifold. We also empirically verify this theoretical prediction in~\S\ref{subsec:toy_example}.

On top of this tangent-space estimation, we design an editing algorithm that operates in a projection loop: take a small step along the current tangent, then use a few denoising steps to project back toward the manifold, refresh the tangent from the updated state, and repeat. A schematic illustration of this procedure is provided in Figure~\ref{fig:algorithm}. The result is stable, interactive control without re-diffusion or Jacobian evaluation. Similar to subspace editing \citep{chen2024exploring}, controls are low-dimensional and semantic; in contrast to fixed Jacobian probes, our basis updates per step and parallelizes naturally across a local ensemble. The explicit tangent frame also makes the procedure optimization-compatible: differentiable objectives on the manifold can drive alternating tangent steps and projections, yielding principled edits while maintaining realism, in the spirit of manifold-preserving guidance and manifold-constrained inverse solvers \citep{he2023manifold,chung2022improving}. Continuous, fine-grained adjustments follow naturally, without expensive restarts. We conduct experiments on both unsupervised edits along the estimated tangent bases and CLIP-score–guided optimization to verify that our tangent-space formulation enables smooth, semantic, and controllable editing. In summary, our contributions are as follows:
\begin{itemize}
\item 
We prove a non-asymptotic bound (Theorem~\ref{thm:tangent-approx}) on the deviation between the true tangent space and the subspace spanned by $k$ ambient secants, which provides a formal justification for using finite ambient secants as a proxy for the intrinsic tangent space.

\item Based on our theoretical guarantee, we propose a manifold-aware editing framework that couples a Jacobian-free tangent estimator with a projection loop and plugs directly into modern diffusion samplers.

\item Extensive experiments demonstrate the resulting tangent space supports efficient, continuous editing with smooth and fine-grained manipulations.
\end{itemize}

\section{Related Works}
Diffusion models exhibit evidence of manifold structure, motivating edits that stay near the data manifold to preserve semantics \cite{fefferman2016testing,stanczuk2024diffusion,kvinge2023exploring}. Prior work reduces off-manifold drift via manifold-constrained sampling and improved guidance, and enables real-image editing through inversion and guidance design \citep{he2023manifold,chung2024cfg++,mokady2023null}. Separately, training-free methods provide inference-time strength control, including noise-level adjustment (SDEdit) and attention manipulation (Prompt-to-Prompt, MasaCtrl), as well as interpolation-based morphing \citep{meng2021sdedit,hertz2022prompt,Cao2023MasaCtrl,zhang2024diffmorpher}. Many approaches still couple edit strength to the sampling path, requiring re-diffusion to change strength. A more detailed discussion of related work is provided in appendix ~\S\ref{sec:related}.
\section{Background: Diffusion Models}
\label{sec:background-diffusion}
Diffusion generative models define a forward noising process and learn its time-reversal to map noise back to data, using either a reverse-time SDE or a deterministic probability–flow ODE \citep{ho2020denoising,song2021score,chen2024exploring,park2023understanding}.

\noindent \textbf{Continuous-time backbone (VP).}
For \(x_0\in\mathbb{R}^d\), the variance-preserving (VP) forward SDE is
\begin{equation}
\label{eq:vp_sde}
\mathrm{d}x \;=\; -\tfrac{1}{2}\beta(t)\,x\,\mathrm{d}t \;+\; \sqrt{\beta(t)}\,\mathrm{d}W_t,\qquad t\in[0,T],
\end{equation}
where \(W_t\) is a standard Wiener process and \(\beta(t)\ge 0\) with \(\int_0^T\beta(s)\,\mathrm{d}s<\infty\).
The conditionals satisfy \(q(x_t\mid x_0)=\mathcal N(\alpha_t x_0,\sigma_t^2 I)\), with
\(\alpha_t=\exp\!\bigl(-\tfrac12\!\int_0^t\beta(s)\,\mathrm{d}s\bigr)\) and \(\sigma_t=\sqrt{1-\alpha_t^2}\).
In particular, \(\alpha_T\to 0\) and \(\sigma_T\to 1\), so \(x_T\sim\mathcal N(0,I)\).
Let \(s_\theta(x,t)\approx\nabla_x\log p_t(x)\) denote the learned score, and let \(\epsilon_\theta(x_t,t)\) denote the noise-prediction network (same shape as \(\epsilon\)).
Under VP we have \(s_\theta(x_t,t)\approx -\,\sigma_t^{-1}\,\epsilon_\theta(x_t,t)\).
We train \(\epsilon_\theta\) with the standard \(\epsilon\)-prediction objective:
\begin{equation}
\label{eq:loss_simple_compact}
\mathcal{L}_{\mathrm{simple}}
= \mathbb{E}\!\left[\left\|\,\epsilon-\epsilon_\theta(x_t,t)\,\right\|_2^2\right],
\end{equation}
with \(x_t=\alpha_t x_0+\sigma_t\epsilon\), \(t\sim\mathrm{Unif}[0,T]\),
\(x_0\sim p_{\text{data}}\), and \(\epsilon\sim\mathcal N(0,I)\).

The reverse sampling process can be done via either the reverse-time SDE  or the equivalent probability flow ODE. The latter is given by
\begin{equation}
\label{eq:pf_ode}
\frac{\mathrm{d}x}{\mathrm{d}t} \;=\; -\tfrac12\beta(t)\,x \;-\; \tfrac12\beta(t)\,s_\theta(x,t).
\end{equation}
With the exact score, the reverse-time SDE and the probability–flow ODE induce the same marginal distributions for all \(t\) \cite{song2021score}.

\noindent \textbf{Time discretization (one-step).}
For a decreasing grid \(T=t_0>t_1>\cdots>t_K=0\) with \(\Delta t_k=t_{k-1}-t_k>0\):
\begin{equation}
\label{eq:pf_step}
x_{k-1} \;=\; x_k + \Bigl[\tfrac12\beta(t_k)\,x_k+\tfrac12\beta(t_k)\,s_\theta(x_k,t_k)\Bigr]\Delta t_k.
\end{equation}
These cover stochastic and deterministic samplers; higher-order solvers and DDIM-style zero-variance updates arise from the same backbone.

\section{Method}
\label{sec:method}

We adopt the data-manifold view and develop a \emph{Jacobian-free editor} for diffusion sampling. Intuitively, we treat the data manifold as the centerline of a thin tube in ambient space and design edits that (i) move \emph{along} the manifold (tangent directions) while (ii) routinely correcting small normal drifts \emph{back into} the tube. This perspective lets us pair geometric control with standard probability–flow (PF) sampling in a modular way.

\noindent \textbf{Pipeline overview.}
We summarize our editor in the following three steps:
\begin{enumerate}
\renewcommand\theenumi{\roman{enumi}}
\renewcommand\labelenumi{(\theenumi)}
\item \textbf{Local tangent estimation at the footpoint.}
Starting from a small ensemble in latent space, we map probe points into the tube and form \emph{ secants}. A principal component analysis (PCA) of these secants yields an orthonormal tangent frame and the associated projector onto the tangent space at the current footpoint.

\item \textbf{Frame transport without Jacobians.}
For any chosen editing direction in the tangent space, we update the frame by transporting the entire ensemble along that direction, thereby avoiding explicit Jacobians while maintaining coherence of the frame along the path.

\item \textbf{Tube-preserving retractions.}
To prevent the frame from drifting off the data manifold, we apply a noising--denoising retraction after each transport step, pulling the transported ensemble back into the tube neighborhood. This stabilizes the procedure, keeps iterates close to the manifold, and curbs the accumulation of off-manifold error.
\end{enumerate}
Throughout we adopt the notation of Sec.~\S\ref{sec:background-diffusion} for the VP score SDE, the PF ODE, their one-step discretizations, and the standard training objective.

\subsection{Tubular Geometry and Probability Flow}
\label{subsec:tube-geometry}

In generative modeling, it is now standard to assume a \emph{data manifold}—that the data distribution is supported on or near a low-dimensional smooth manifold. Existing theoretical and empirical results further indicate that diffusion models adapt to this manifold \cite{tang2024adaptivity, kamkari2024geometric}, concentrating probability mass on and around it. Guided by this view, we represent the geometry by a tubular neighborhood of radius \(\rho\) around a smooth manifold \(\mathcal{M}\), which captures the structure while allowing a controlled normal slack to accommodate small off-manifold deviations.

\textbf{Geometry of the tube.}
We model the support as a smooth, embedded \(d\)-dimensional manifold \(\mathcal M\subset\mathbb R^n\) with \(d\ll n\) and positive reach \(\mathrm{reach}(\mathcal M)>0\). For any \(\rho\in(0,\mathrm{reach}(\mathcal M))\), the (closed) tubular neighborhood
$\mathcal T_\rho := \{\,x\in\mathbb R^n : \mathrm{dist}(x,\mathcal M)\le \rho \,\}$
admits a unique nearest-point projection \(\pi:\mathcal T_\rho\to\mathcal M\) that is \(C^2\). Every \(x\in\mathcal T_\rho\) decomposes as:
\[
x = x_\parallel + s(x),\qquad x_\parallel := \pi(x)\in\mathcal M, \qquad
s(x) := x-\pi(x)\in N_{x_\parallel}\mathcal M .
\]
where \(s(x)\) (the \emph{slack}) measures normal deviation.

\textbf{Probability flow and data manifold.}
Let \(u_\theta(x,t)\) denote the probability–flow vector field and \(\varphi_t\) its one-parameter flow with \(\varphi_0=\mathrm{id}\).
For \(T>0\), define \(\varphi_{-T}:=\varphi_T^{-1}\).
In generation we sample \(x_T \sim p_T \approx \mathcal N(0,I)\) and compute
\(x_0 = \varphi_{-T}(x_T)\) by integrating \(\dot x = u_\theta(x,t)\) backward in time from \(t=T\) to \(t=0\); thus the PF flow carries points from terminal noise to the data region.

\begin{assumption}[PF endpoints lie in the tube]\label{as:pf-to-tube}
For terminal noise \(z\sim p_T\), the inverse PF endpoint satisfies
\(\varphi_{-T}(z)\in\mathcal T_\rho\) for \(p_T\)-almost every \(z\), i.e., \((\varphi_{-T})_\# p_T\) is supported on \(\mathcal T_\rho\).
\end{assumption}

Under Assumption~\ref{as:pf-to-tube}, define the backward PF map and the associated composition with the nearest-point projection \(\pi\), which is defined as:
\begin{equation}
\label{eq:phi-pi}
\begin{aligned}
\Phi &: \mathbb{R}^n \to \mathcal{T}_\rho, && \Phi(z) := \varphi_{-T}(z),\\
\Pi  &: \mathbb{R}^n \to \mathcal{M},      && \Pi(z)  := \pi\!\bigl(\Phi(z)\bigr).
\end{aligned}
\end{equation}
The differential \(\mathrm{D}\Pi(z): \mathbb{R}^n \!\to\! T_{\Pi(z)}\mathcal{M}\) provides tangent directions at \(\Pi(z)\), while the normal slack
\(s(z) := \Phi(z) - \Pi(z) \in N_{\Pi(z)}\mathcal{M}\) quantifies the off-manifold deviation.

A canonical route to estimate the tangent space is to take the SVD of the Jacobian of the ambient map (e.g., \( \mathrm D\Phi(z) \) or \( \mathrm D\Pi(z) \)), whose leading left singular vectors span the local data and thereby recover tangent directions. However, \(\Phi\) is realized as a long composition of denoising steps, so \( \mathrm D\Phi(z) \) is a product of per-step Jacobians; even accessing it via Jacobian–vector or adjoint–Jacobian–vector products requires backpropagating through all steps, which is prohibitively expensive in compute and memory. Recent work \cite{park2023understanding} suggests using the Jacobian of the denoising network at the final (near-data) step instead, avoiding the full unrolled chain, but this remains costly:  it still requires forming an explicit  Jacobian, and the subsequent SVD operates at a large scale, pushing compute and memory to prohibitive levels at high resolution. Motivated by these costs, we introduce a Jacobian-free tangent-space estimator based on ensembles generated by the probability flow ODE.

\subsection{Estimating the Tangent}
\label{subsec:tangent-estimation}
To perform on-manifold edits, we need a local tangent projector at a near-manifold point \(x\) while avoiding Jacobians of \(\Phi\) or the score network. Our estimator constructs this projector from a flow-generated ensemble: small latent perturbations are transported by the probability flow, and the principal directions of the resulting ensemble (via PCA) approximate the tangent space \(T_{x_\parallel}\) (at \(x_\parallel=\pi(x)\)).

\textbf{Generating ensemble.}
Given initial noise \(z\in\mathbb{R}^n\), set \(x=\Phi(z)\in\mathcal T_\rho\) and \(x_{\parallel}=\Pi(z)=\pi(x)\in\mathcal M\).
Draw i.i.d.\ perturbations \(\{\xi_i\}_{i=1}^m\sim\mathcal N(0,I_n)\) and fix a perturbation scale \(\sigma>0\).
Define the perturbed latents and their flow images by:
\[
z_i := z + \sigma\,\xi_i,\qquad x_i := \Phi(z_i),
\]
and the one-sided secants are as 
$
\Delta_i := x_i - x .
$
Then the sample mean and empirical covariance of the secants \(\{\Delta_i\}_{i=1}^m\) are computed as:
\[
\bar\Delta := \tfrac{1}{m}\sum_{i=1}^m \Delta_i, \quad
\widehat C_{x} := \tfrac{1}{m}\sum_{i=1}^m(\Delta_i-\bar\Delta)(\Delta_i-\bar\Delta)^\top .
\]
Let \(\widehat U_k\in\mathbb{R}^{n\times k}\) collect the eigenvectors associated with the \(k\) largest eigenvalues of \(\widehat C_{x}\) (with orthonormal columns).
We estimate the tangent space at \(x_\parallel\) by \(\widehat T_{x_\parallel} := \operatorname{span}(\widehat U_k)\), and the associated orthogonal projector by \(\widehat P_{x_\parallel} := \widehat U_k \widehat U_k^\top\).
For efficiency, instead of eigendecomposing \(\widehat C_x\in\mathbb{R}^{n\times n}\), form the centered secant matrix \(D:=[\,\Delta_1-\bar\Delta,\ldots,\Delta_m-\bar\Delta\,]\in\mathbb{R}^{n\times m}\) and compute the reduced SVD \(D=\widehat U\,\widehat\Sigma\,\widehat V^\top\).
It leads to a much cheaper computation of \(\widehat U_k := \widehat U_{:,1:k}\) when \(m\!\ll\! n\).

Compared with antithetic pairs \(z_i^{\pm}:= z \pm \sigma\,\xi_i\), one-sided probes \(z_i:= z + \sigma\,\xi_i\) require roughly half as many probability–flow evaluations and simplify scheduling; in exchange, the displacement incurs a second-order truncation error:
\begin{lemma}[Truncation error for \(\Phi\)]
Assume \(\Phi\in C^3\) in a neighborhood of \(z\).
Then for any \(\xi\in\mathbb R^n\),
\[
\Phi(z+\sigma\xi)-\Phi(z)
= \sigma\,\mathrm D\Phi(z)\,\xi
\;+\; \tfrac{\sigma^2}{2}\,\mathrm D^2\Phi(z)[\xi,\xi]
\;+\; R,
\]
where \(\mathrm D^2\Phi(z)[\cdot,\cdot]\) denotes the bilinear action of the Hessian and 
\(R\) represents a higher order residual.
\end{lemma}

Next we provide an error estimate of the approximated tangent space and tangent projector.

Since $x\in\mathcal T_\rho$, there is a unique footpoint $x_\parallel:=\pi(x)\in\mathcal M$.
We obtain the tangent space at $x_\parallel$ and its orthogonal projector purely from
the Jacobian $D\Pi(z)$:
\[
\begin{aligned}
M_\star &:= D\Pi(z)\,D\Pi(z)^\top = U\,\Lambda\,U^\top \quad (\lambda_1 \ge \cdots \ge \lambda_n),\\
U_\star &:= U_{[:,1:d]}, \qquad P_\star := U_\star U_\star^\top .
\end{aligned}
\]
If $\operatorname{rank} D\Pi(z)=d$, then $\operatorname{span}(U_\star)=T_{x_\parallel}\mathcal M$
and $P_\star$ is the orthogonal projector onto $T_{x_\parallel}\mathcal M$.

In what follows, we compare the tangent spaces induced by the differential target
\(M_\star := D\Pi(z)D\Pi(z)^\top\) with those recovered from the scaled \emph{empirical}
secant covariance \(\sigma^{-2}\widehat C_x\), computed from \(m\) probability–flow
(PF)–propagated perturbations \(z_i=z+\sigma\xi_i\) with \(x_i=\Phi(z_i)\) and
secants \(\Delta_i=x_i-x\). This comparison quantifies the gap between the
implementable estimator and its infinitesimal (differential) counterpart.

\noindent \textbf{Tangent recovery guarantee.} When $k\le d$, we can directly bound how far the subspace spanned by $k$ ambient secants deviates from the true tangent space: the deviation shrinks with stronger tangential signal and grows only with the small truncation and tube-geometry effects. Using symmetric probes further reduces this deviation. The theorem below formalizes this statement.

\begin{theorem}[Subspace deviation from $k$ ambient secants]
\label{thm:tangent-approx}
Fix $k\le d$ and perturbations $\{\xi_i\}_{i=1}^k$, and write $\Xi=[\xi_1,\ldots,\xi_k]$.
Let $T:=T_{x_\parallel}\mathcal M$, $S_k:=\mathrm{span}\{\Delta_1,\ldots,\Delta_k\}$, and
$s_{\min}:=\sigma_{\min}(D\Pi(z)\,\Xi)>0$.
Assume in addition that $
\|P_T^\perp D\Phi(z)\|_2 \;\le\; L_\perp\,\rho$,
and the normal second–order deviation is curvature–controlled:
$\|P_T^\perp R\|_2\le C_{\mathrm{curv}}\kappa\,\sigma^2\|\Xi\|_2^2$,
where $R:=[R_1,\ldots,R_k]$ with $R_i:=\tfrac{\sigma^2}{2}\,D^2\Phi(z)[\xi_i,\xi_i]$.
Then, for sufficiently small $\sigma,\rho$,
\[
\|P_T^\perp P_{S_k}\|_2
\;\le\;
\frac{L_\perp\,\rho\,\|\Xi\|_2}{s_{\min}}
\;+\;
\frac{C_{\mathrm{curv}}\,\kappa\,\sigma\,\|\Xi\|_2^2}{s_{\min}}
\;+\;
\frac{C_3\,\sigma^2\,\|\Xi\|_2^3}{s_{\min}}.
\]
In particular, the first two terms vanish as $\rho,\kappa\to0$ for fixed small $\sigma$. 
\end{theorem}

\subsection{Algorithm}
\label{subsec:algorithm}

\begin{wrapfigure}{r}{0.48\columnwidth}
\vspace{-50pt}
\begin{minipage}{0.48\columnwidth}

\begin{algorithm}[H]
\caption{GeoEdit}
\label{alg:geoedit}
\begin{algorithmic}[1]
\Require PF map $\Phi$, guidance $g(\cdot)$, initial latent $z$, ensemble size $m$, perturbation scale $\sigma$, step size $\eta$, iterations $K$, retraction operator $\mathsf{Retract}(\cdot)$, refresh period $q$
\Ensure Edited sample $x$
\State $x \gets \Phi(z)$
\State Sample $\{\xi_i\}_{i=1}^m \sim \mathcal N(0,I)$; set $z_i \gets z+\sigma\xi_i$ and $x_i \gets \Phi(z_i)$ for $i=1,\dots,m$
\For{$k=1$ to $K$}
  \If{$k=1$ \textbf{or} $k \bmod q = 0$} 
    \State $\Delta_i \gets x_i - x$
    \State $\bar\Delta \gets \tfrac{1}{m}\sum_{i=1}^m \Delta_i$
    \State $D \gets [\Delta_1-\bar\Delta,\ldots,\Delta_m-\bar\Delta]$
    \State Compute SVD $D=\widehat U\widehat\Sigma\widehat V^\top$
    \State $\widehat P \gets \widehat U_d \widehat U_d^\top$
  \EndIf
  \State $v \gets \widehat P\, g(x)$
  \State $x \gets x+\eta v$
  \State $x_i \gets x_i+\eta v$
  \State $(x,\{x_i\}) \gets \mathsf{Retract}(x,\{x_i\})$
\EndFor
\State \Return $x$
\end{algorithmic}
\end{algorithm}

\end{minipage}
\vspace{-30pt}
\end{wrapfigure}

We instantiate \emph{GeoEdit}, turning the estimated tangent space into an edit operator that steers updates along the manifold while suppressing normal drift. From a latent $z$, set $x=\Phi(z)$; draw small perturbations $z_i=z+\sigma\xi_i$, propagate to $x_i=\Phi(z_i)$, form secants $\Delta_i=x_i-x$, and run PCA to obtain $\widehat U_d$ and the projector $\widehat P_{x_\parallel}:=\widehat U_d\widehat U_d^\top$. Given an ambient edit direction $g(x)$, project $v=\widehat P_{x_\parallel}g(x)$ and take a short step $x\leftarrow x+\eta v$ (and $x_i\leftarrow x_i+\eta v$). To remain within $\mathcal T_\rho$, interleave these transports with a few noising--denoising retractions on $\{x,x_i\}$, refreshing $\widehat U_d$ as needed. The concentration of the PCA basis around $T_{x_\parallel}\mathcal M$ and the accuracy of $\widehat P_{x_\parallel}$ follow from \Cref{thm:tangent-approx}. We provide a detailed algorithm in Algorithm \ref{alg:geoedit}.


GeoEdit is compatible with Jacobian-based starts. If a Jacobian is available near $x=\Phi(z)$
(e.g., an estimate of $D\Phi(z)$ or the network Jacobian at the final step),
initialize the tangent basis by taking the top-$d$ left singular vectors $U_0$ and set
$\widehat U_d \leftarrow U_0$ and $\widehat P_{x_\parallel}\leftarrow \widehat U_d\,\widehat U_d^\top$ at the first step.
To avoid further Jacobian evaluations, instantiate $d$ small anchors aligned with $\widehat U_d$,
e.g., $x^{(j)}:=x+\varepsilon\,\widehat u_j$ for the columns $\widehat u_j$ of $\widehat U_d$,
and carry $\{x^{(j)}\}_{j=1}^d$ forward with the same projected update
$v=\widehat P_{x_\parallel}g(x)$ and the same noising--denoising retractions as in Alg.~\ref{alg:geoedit}.
At refresh steps, recompute $\widehat U_d$ by PCA on the ensemble secants $\Delta_i=x_i-x$.

\subsubsection{Projected Gradient Descent on the Data Manifold}
Given a smooth loss $f:\mathbb R^n\!\to\!\mathbb R$ and a current iterate $x$, compute the ambient gradient $g=\nabla f(x)$ and its tangent projection $v=\widehat P_{x_\parallel} g$. Take a projected step and retract:
\[
\tilde x = x - \eta\, v, \qquad x^{+} = \mathsf{Retract}(\tilde x),
\]
where $\mathsf{Retract}$ is a short noising--denoising move back into $\mathcal T_\rho$. This approximates the Riemannian gradient descent on $\mathcal M$ using the nearest-point retraction; with the empirical projector $\widehat P_{x_\parallel}$, \Cref{thm:tangent-approx} bounds the induced subspace error.

\section{Experiments}
We assess the low-rank structure of our tangent estimator in subsection~\S\ref{subsec:toy_example} and visualize tangent-space editing directions  in subsection~\S\ref{subsec:tangent-editing}. On real image diffusion models (CelebA-HQ \cite{liu2015deep}, LSUN-church \cite{yu2015lsun}, and latent Stable Diffusion  \cite{rombach2022high}), we visualize unsupervised edits along the estimated tangent bases and obtain smooth, semantic traversals (Section~\ref{subsec:tangent-editing}). Finally, we plug the method into CLIP-guided \cite{radford2021learning} pipelines as an extended setting, demonstrating prompt-conditioned optimization editing.

\textbf{Implementation Details.} Unless otherwise noted, all experiments are run on a single NVIDIA RTX PRO 6000 GPU. We use $5$ short noising--denoising steps in $\mathsf{Retract}$ and perform editing at diffusion time $t = 0.2T$ (with $T$ the total number of steps), leaving a brief refinement window. Each editing trajectory runs for $64$ iterations with a step size $\eta \in [5,10]$. We find that the perturbation strength can be set to $\sigma=0.2$ across models and datasets, and use this value as the default in all experiments. 
Empirically, this choice yields stable local frame estimation while remaining robust across different generation settings. Additional hyperparameters are provided in the Appendix ~\S\ref{sec:appendix}. 


\subsection{Toy Examples}
\label{subsec:toy_example}

\begin{figure}[!t]
\centering
\begin{minipage}[t]{0.49\columnwidth}
  \centering
  \includegraphics[width=\linewidth,height=0.23\textheight,keepaspectratio]{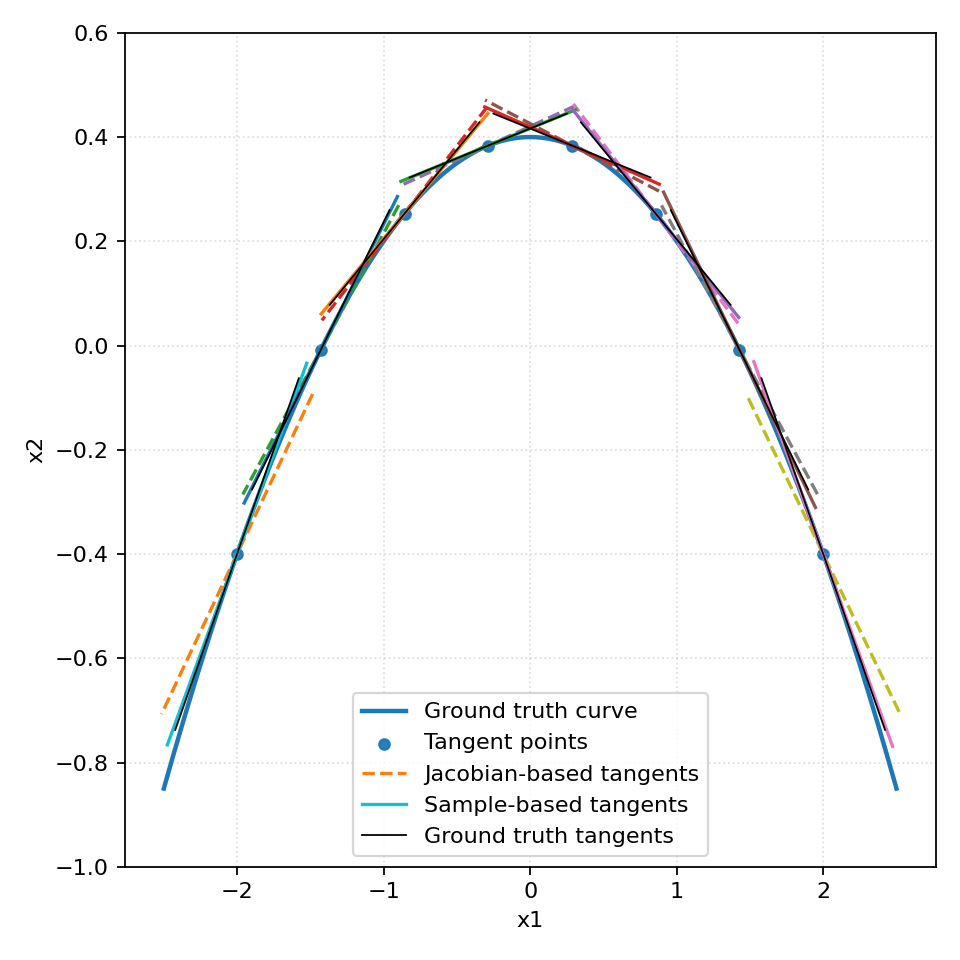}
  \captionof{figure}{Comparison of tangent lines.}
  \label{fig:tangent_compare}
\end{minipage}\hfill
\begin{minipage}[t]{0.49\columnwidth}
  \centering
  \includegraphics[width=\linewidth,height=0.28\textheight,keepaspectratio]{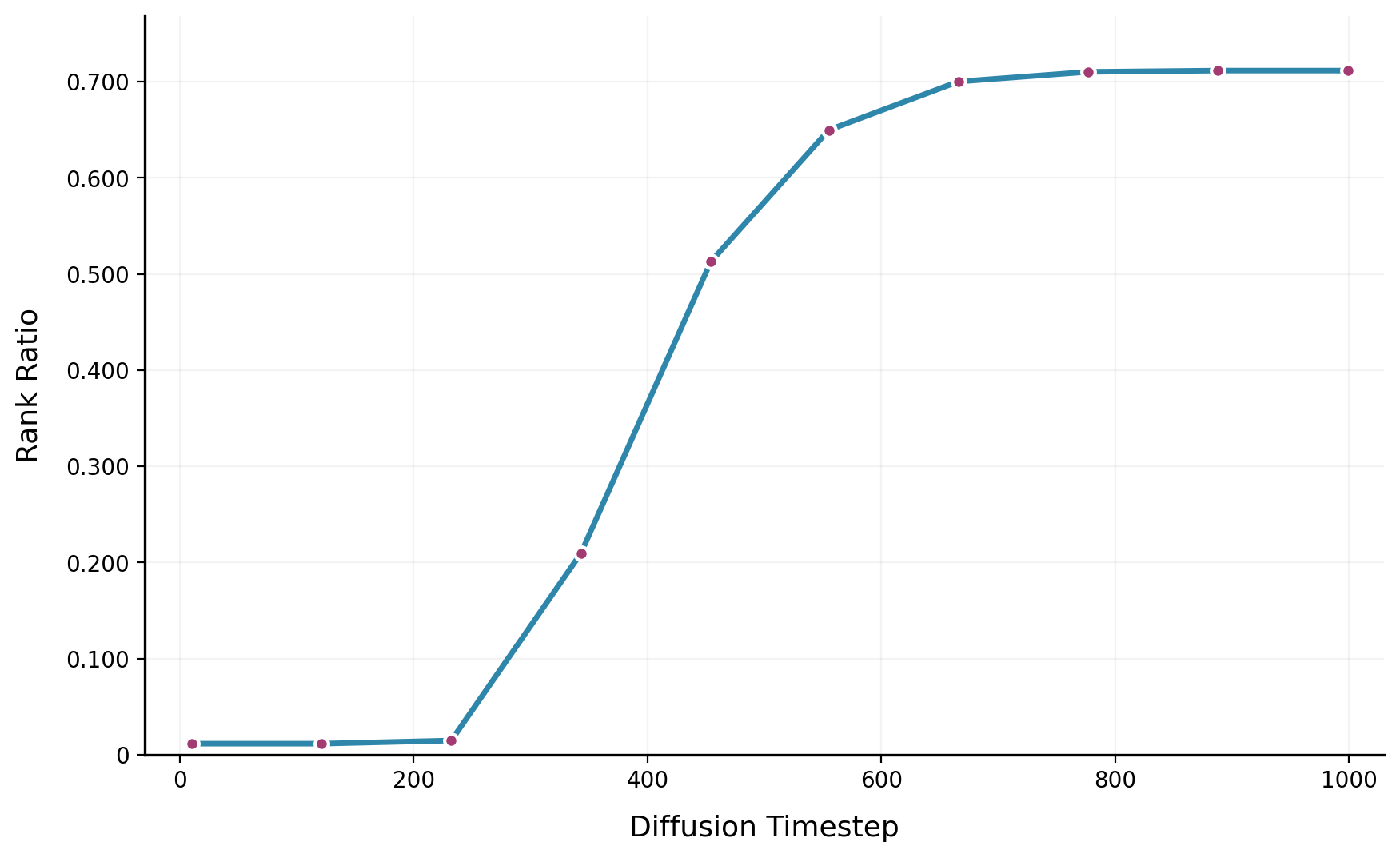}
  \captionof{figure}{Rank ratio on CIFAR-10.}
  \label{fig:rank-cifar}
\end{minipage}
\end{figure}

We empirically verify our method on a toy problem where the data manifold is generated on a curve in 2D. Experiments consists of estimating the tangent lines with different methods including the Jacobian based method and our sample based method. The sampling trajectories following their respective score functions.

The target data distribution $p(x_0)$ consists of samples distributed along a parabola. We trained
a conditional diffusion model using a small neural network. For the Jacobian based tangent space estimation, we follow the work \cite{chen2024exploring} to compute the Jacobian of the posterior mean predictor:
$
\hat{x}_{0, t}=f_{{\theta}, t}\left(x_t \right):=\frac{x_t-\sqrt{1-\alpha_t} {\epsilon}_{{\theta}}\left({x}_t, t\right)}{\sqrt{\alpha_t}}.
$
The top left singular vector of ${J}_{{\theta}, t}\left({x}_t\right)=\nabla_{{x}_t} {f}_{{\theta}, t}\left({x}_t\right)$ serves as the estimation of tangent direction. For our sample based tangent direction we follow Algorithm~\ref{alg:geoedit} to obtain $\widehat U_1$. We pull the tangent estimations at the same starting sample before inversion and plot the ground truth tangent together with the estimated tangents in Fig. \ref{fig:tangent_compare}. According to the result, sample-based tangent estimation aligns better with the ground truth tangent.

We also evaluate a PCA-based \emph{rank ratio} on concatenated, mean-centered samples to quantify local low dimensionality near the end of generation. 
Let $X\in\mathbb{R}^{n\times N}$ be the stacked data (columns mean-centered) with singular values $\sigma_1\ge\cdots\ge\sigma_{d_{\max}}$ from the SVD $X=U\Sigma V^\top$, where $d_{\max}:=\min(n,N)$. 
For a variance threshold $\eta\in(0,1)$, define the minimal dimension:
\begin{equation*}
    r^\star \;=\; \min\left\{\,r\in\{1,\ldots,d_{\max}\}:\ 
    \frac{\sum_{i=1}^{r}\sigma_i^{2}}{\sum_{i=1}^{d_{\max}}\sigma_i^{2}} \ \ge\ \eta\ \right\}.
\end{equation*}
The (normalized) rank ratio is then
$\mathrm{Rank~ratio}\;=\;\frac{r^\star}{d_{\max}}$, 
which is small when variance is concentrated in a low-dimensional subspace. We illustrate this analysis on CIFAR-10. Each image has ambient dimension \(d=32\times32\times3=3072\), and we use \(N=32\times32\times3\) samples to form the stacked, mean-centered matrix \(X\in\mathbb{R}^{d\times N}\). With \(\eta=0.99\), we compute \(r^\star\) from the PCA spectrum of \(X\) and report the  rank ratio \(r^\star/d_{\max}\) with \(d_{\max}=\min(n,N)=N\). As shown in Fig.~\ref{fig:rank-cifar}, the estimated intrinsic dimension decreases along the denoising trajectory and reaches a minimum for diffusion timesteps \(\le 200\), supporting a locally low-rank structure near the end of generation.


\subsection{Editing with the Estimated Tangent Space}
\label{subsec:tangent-editing}

\captionsetup[subfigure]{skip=3pt, belowskip=0pt}

\begin{figure*}[!t]
\centering

\begin{minipage}[t]{0.49\textwidth}\vspace{0pt}
\centering

\begin{subfigure}[t]{\linewidth}
  \centering
  \includegraphics[width=\linewidth]{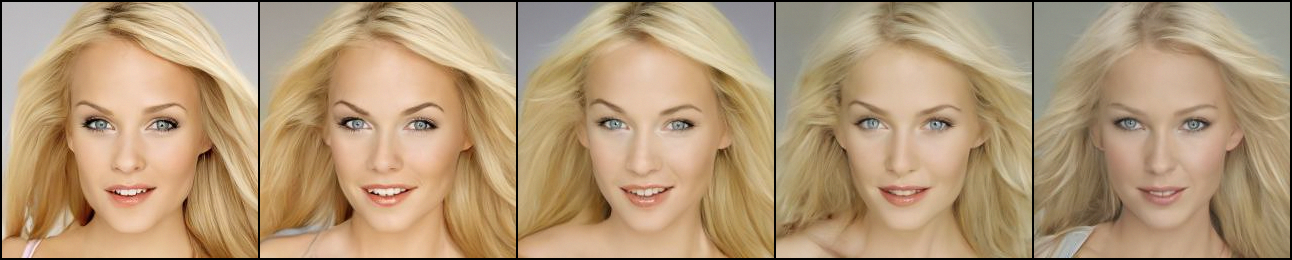}
  \caption{Editing along $\widehat u_1$.}
\end{subfigure}\par

\begin{subfigure}[t]{\linewidth}
  \centering
  \includegraphics[width=\linewidth]{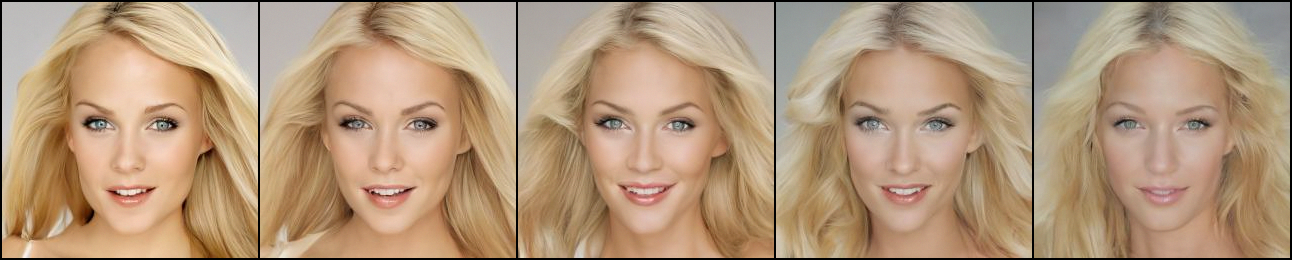}
  \caption{Editing along $\widehat u_2$.}
\end{subfigure}\par

\begin{subfigure}[t]{\linewidth}
  \centering
  \includegraphics[width=\linewidth]{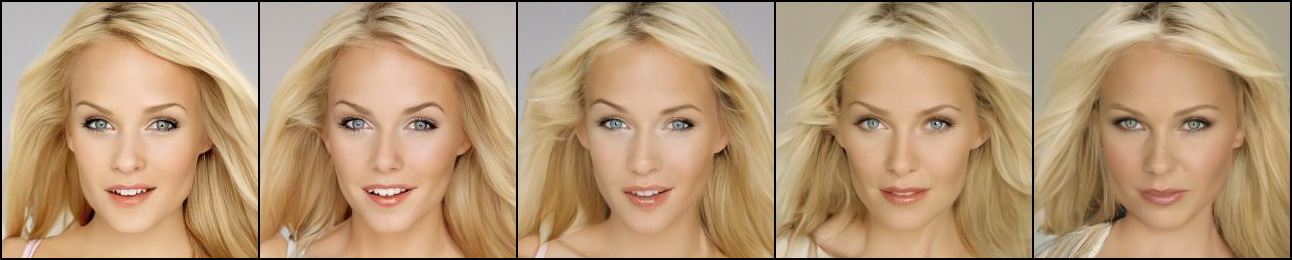}
  \caption{Editing along $\widehat u_3$.}
\end{subfigure}

\captionof{figure}{Edits along different tangent bases. The DDPM is pretrained on CelebaA-HQ.}
\label{fig:face_basis}
\end{minipage}
\hfill
\begin{minipage}[t]{0.49\textwidth}\vspace{0pt}
\centering

\begin{subfigure}[t]{\linewidth}
  \centering
  \includegraphics[width=\linewidth]{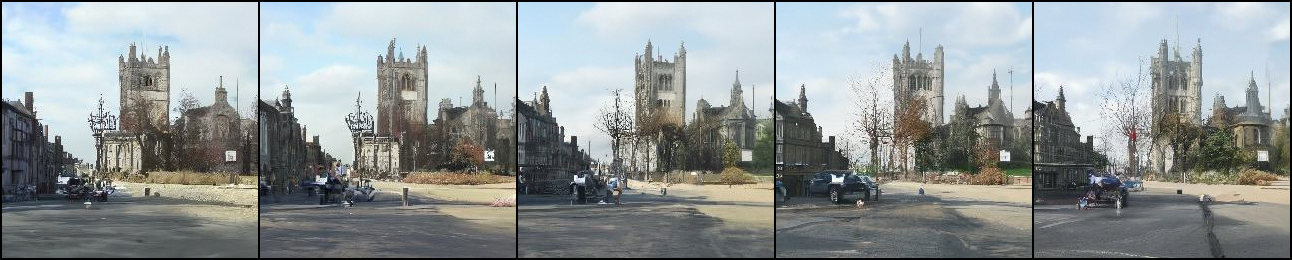}
  \caption{Editing along $\widehat u_1$.}
\end{subfigure}\par

\begin{subfigure}[t]{\linewidth}
  \centering
  \includegraphics[width=\linewidth]{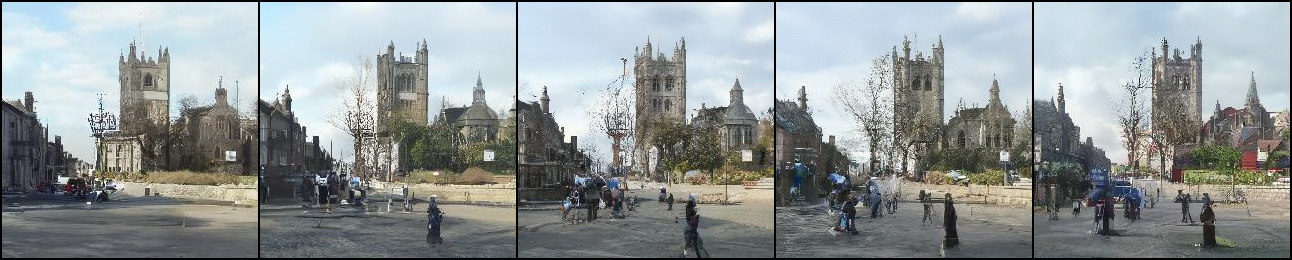}
  \caption{Editing along $\widehat u_2$.}
\end{subfigure}\par

\begin{subfigure}[t]{\linewidth}
  \centering
  \includegraphics[width=\linewidth]{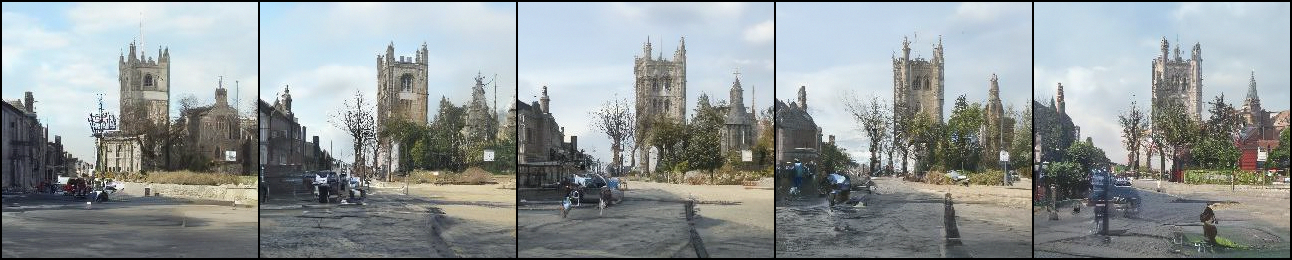}
  \caption{Editing along $\widehat u_3$.}
\end{subfigure}

\captionof{figure}{Edits along different tangent bases. The DDPM model is pretrained on LSUN-church.}
\label{fig:church_basis}
\end{minipage}

\end{figure*}















We demonstrate on-manifold image editing using the estimated tangent basis and projector from Section \ref{subsec:tangent-estimation}. For $x=\Phi(z)\in\mathcal T_\rho$ with footpoint $x_\parallel=\pi(x)\in\mathcal M$, our estimator returns an orthonormal basis $\widehat U_k(x_\parallel)\in\mathbb R^{n\times k}$ and the projector $\widehat P_{x_\parallel}=\widehat U_k\widehat U_k^\top \approx \text{Proj}_{T_{x_\parallel}\mathcal M}$. 

We visualize \emph{unsupervised} edits by selecting three basis directions $\widehat u_1,\widehat u_2,\widehat u_3$ from $\widehat U_k$ and traversing each independently from the same image $x$. Fig. ~\ref{fig:face_basis} and Fig.~\ref{fig:church_basis} organizes the results by (row) editing basis index and (column) editing step. Increasing the number  of steps along a fixed $\widehat u_i$ produces smooth, near-monotone morphs aligned with a single coherent semantic (e.g., color tone, illumination, pose, expression, texture/detail density, background layout), while different tangent directions induce distinct morphing behaviors.

\subsection{CLIP-guided Optimization on the Tangent Space.} 

\begin{figure*}[t]
\centering

\begin{minipage}[t]{0.49\textwidth}
\centering
\captionsetup{type=figure,width=\linewidth}

\begin{subfigure}{\linewidth}
\includegraphics[width=\linewidth]{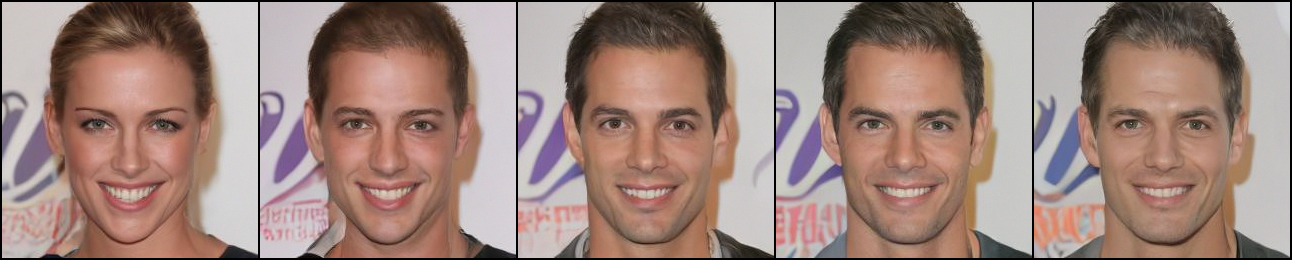}
\caption{Prompt: ``a male''.}
\end{subfigure}

\begin{subfigure}{\linewidth}
\includegraphics[width=\linewidth]{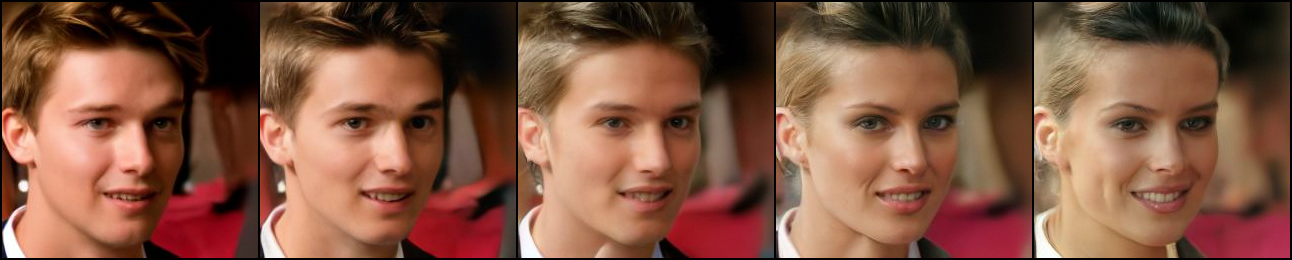}
\caption{Prompt: ``a female''.}
\end{subfigure}

\begin{subfigure}{\linewidth}
\includegraphics[width=\linewidth]{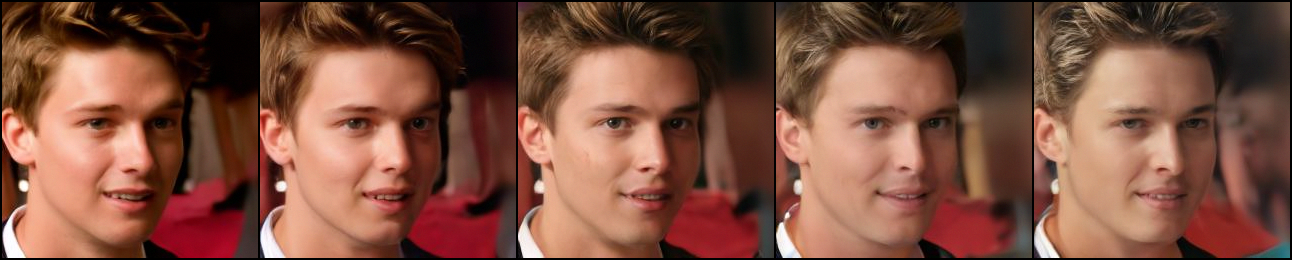}
\caption{Prompt: ``an old male''.}
\end{subfigure}

\caption{CLIP guided image editing. The DDPM is pretrained on CelebaA-HQ.}
\label{fig:face_clip}

\end{minipage}
\hfill
\begin{minipage}[t]{0.49\textwidth}
\centering
\captionsetup{type=figure,width=\linewidth}

\begin{subfigure}{\linewidth}
\includegraphics[width=\linewidth]{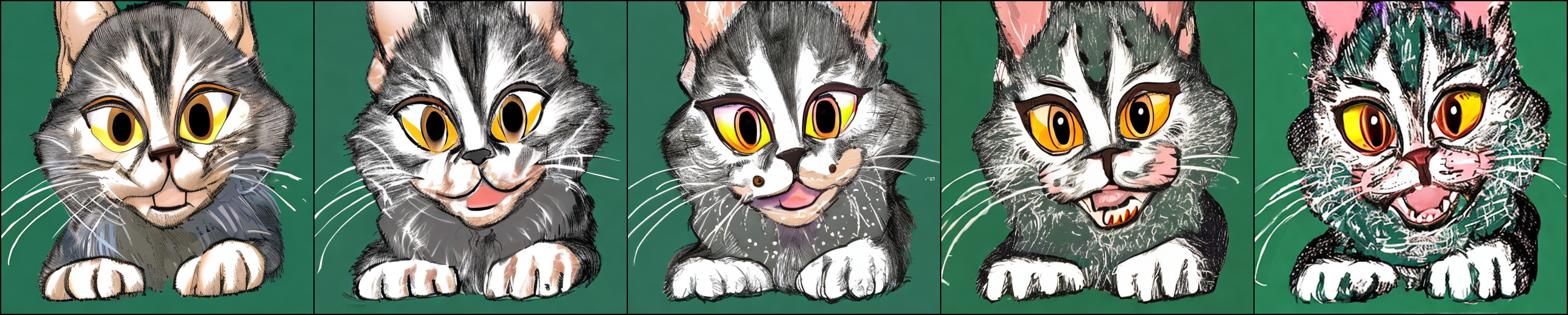}
\caption{Editing along $\widehat u_1$.}
\end{subfigure}

\begin{subfigure}{\linewidth}
\includegraphics[width=\linewidth]{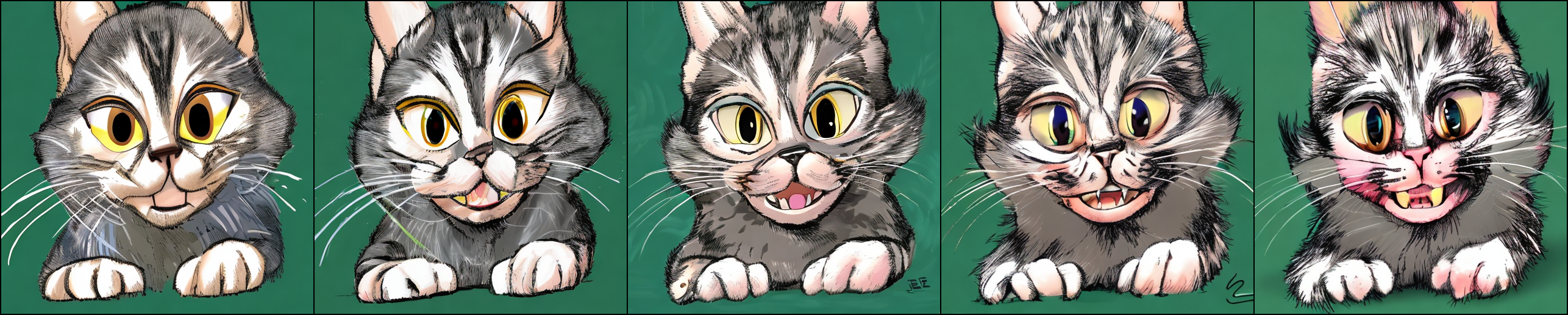}
\caption{Editing along $\widehat u_2$.}
\end{subfigure}

\begin{subfigure}{\linewidth}
\includegraphics[width=\linewidth]{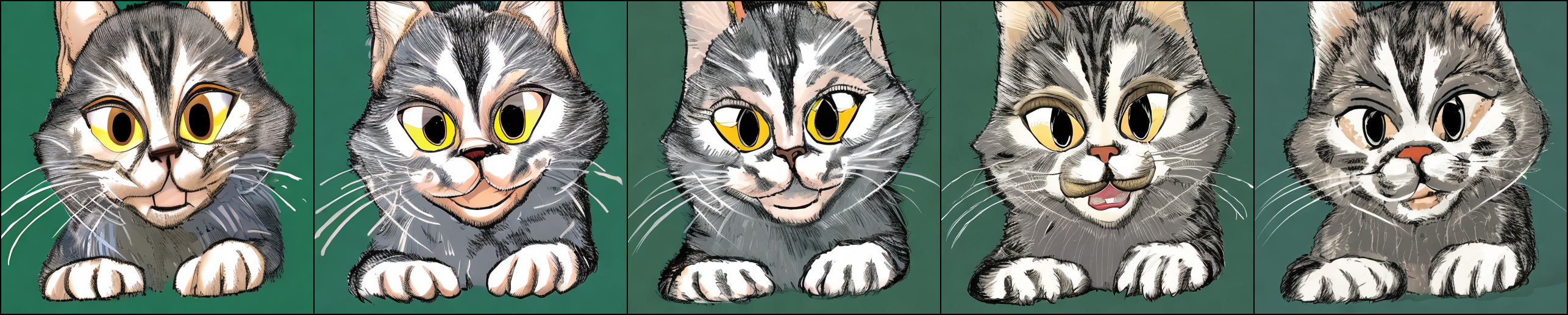}
\caption{Editing along $\widehat u_3$.}
\end{subfigure}
\caption{Unsupervised edit results with Stable Diffusion model.}
\label{fig:sd_basis}
\end{minipage}
\end{figure*}

We choose CLIP guidance to demonstrate the projected manifold optimization algorithm. For prompt-guided editing, we maximize the CLIP cosine similarity $S_{\mathrm{CLIP}}(x,p)=
    \frac{\langle \varphi_{\mathrm{img}}(x),\,\varphi_{\mathrm{text}}(p)\rangle}
         {\|\varphi_{\mathrm{img}}(x)\|_2\,\|\varphi_{\mathrm{text}}(p)\|_2}$ by projected ascent confined to the estimated tangent space:
\begin{equation*}
    x_{t+1}=x_t+\eta\,\widehat P_{x_\parallel^{(t)}}\,\frac{v_t}{\|v_t\|_2},
    \qquad
    v_t=\nabla_x S_{\mathrm{CLIP}}(x_t,p).
    \label{eq:proj-clip}
\end{equation*}
Fig.~\ref{fig:face_clip} illustrates gradual targeted edits for diverse prompts that improve semantic alignment while preserving layout and identity. The procedure requires no annotations and no fine-tuning of the diffusion backbone; CLIP serves purely as an inference-time scoring function and can be replaced by alternative text--image encoders (e.g., SigLIP \cite{zhai2023sigmoid}) without retraining. By avoiding labeled supervision and auxiliary heads, our pipeline reduces exposure to dataset and labeling biases commonly introduced by annotation, and its plug-and-play design enables swapping or auditing the guidance model to mitigate model-specific biases without altering the generator. We use CLIP at $224\times 224$ with standard preprocessing and backpropagate through the image encoder to obtain $v_t$; we run $48$ iterations of \eqref{eq:proj-clip} with $\eta\in[5,8]$.

Our method is directly applicable to latent diffusion models. We instantiate it with Stable Diffusion and report results of GeoEdit in the latent space. Fig.~\ref{fig:sd_basis} shows \emph{unsupervised} edits obtained by traversing three basis directions $\widehat u_1,\widehat u_2,\widehat u_3$ from the same seed. All images are generated with the text prompt \texttt{"cat, cartoon"}. Varying the number of editing steps along each $\widehat u_i$ produces smooth, near-monotone morphs, and different basis directions exhibit distinct semantics. We also use text-to-image diffusion model SDXL-Lightning~\cite{lin2024sdxl} with 4-step generation and visualize edits along the first principal direction of the estimated local frame. Fig.~\ref{fig:sdxl_diverse} shows two examples using the prompts ``cat'' and ``dragon''. 
For ``cat'', the edit produces a smooth change in face orientation. 
For ``dragon'', the edit gradually changes the color from green to gray while simultaneously shifting the rendering toward a sketch-like style. 
These results suggest that GeoEdit captures meaningful local semantic directions.

In addition to fine-scale morphing, we obtain stronger prompt-guided edits by enabling the edit prompt ($p_{\text{edit}}$, “for-prompt”) or other state-of-the-art guidance. The tangent-space estimator and projector remain as in Algorithm~\ref{alg:geoedit}. These mechanisms are complementary and can be composed: we first generate with prompt guidance, then apply localized refinements in the estimated tangent frame around the resulting sample. Fig.~\ref{fig:dogs} illustrates this composition, where prompt guidance sets the coarse semantic target and the tangent-frame traversal provides fine-scale control. 








\begin{figure}[t]
  \centering

  \begin{minipage}[t]{0.48\columnwidth}
    \vspace{0pt}
    \centering
    \setlength{\tabcolsep}{0pt}
    \resizebox{\linewidth}{!}{%
      \begin{tabular}{ccc}
      \includegraphics[width=0.33\linewidth]{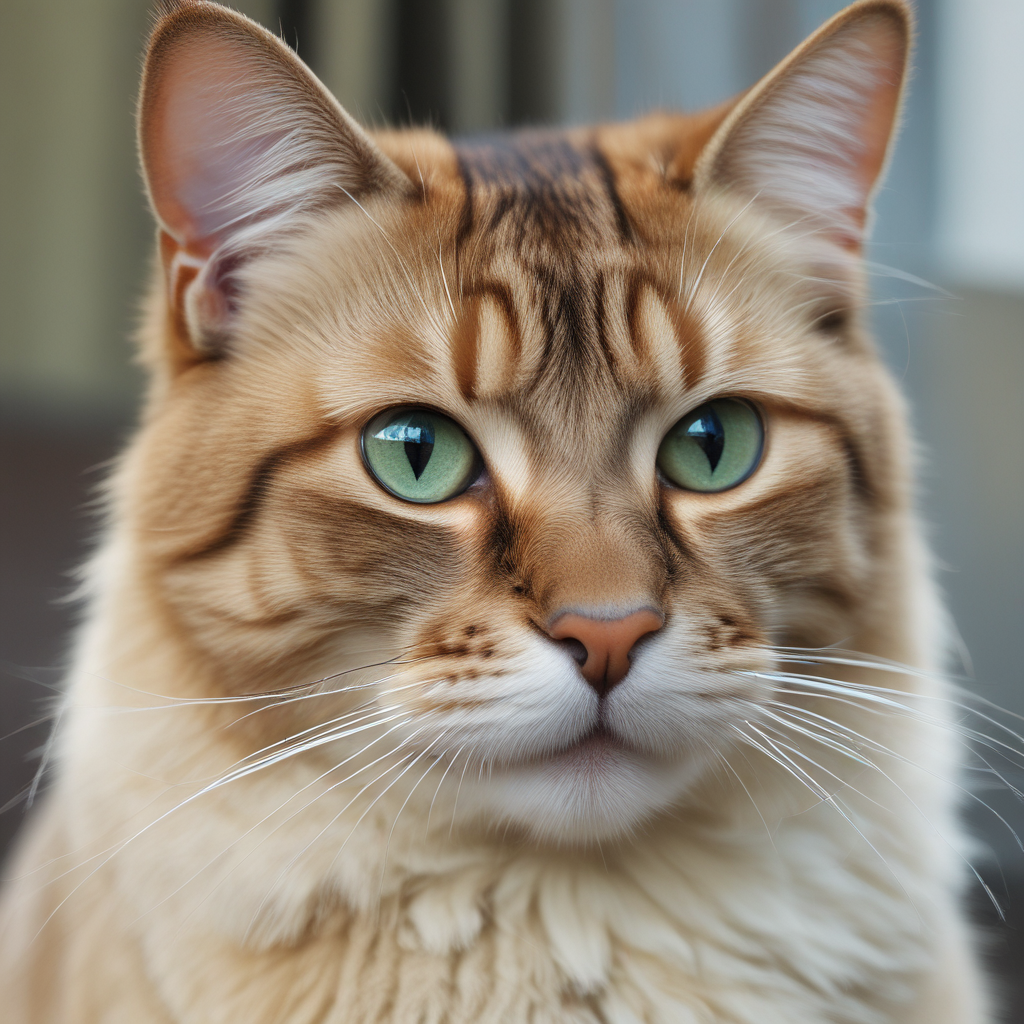} &
      \includegraphics[width=0.33\linewidth]{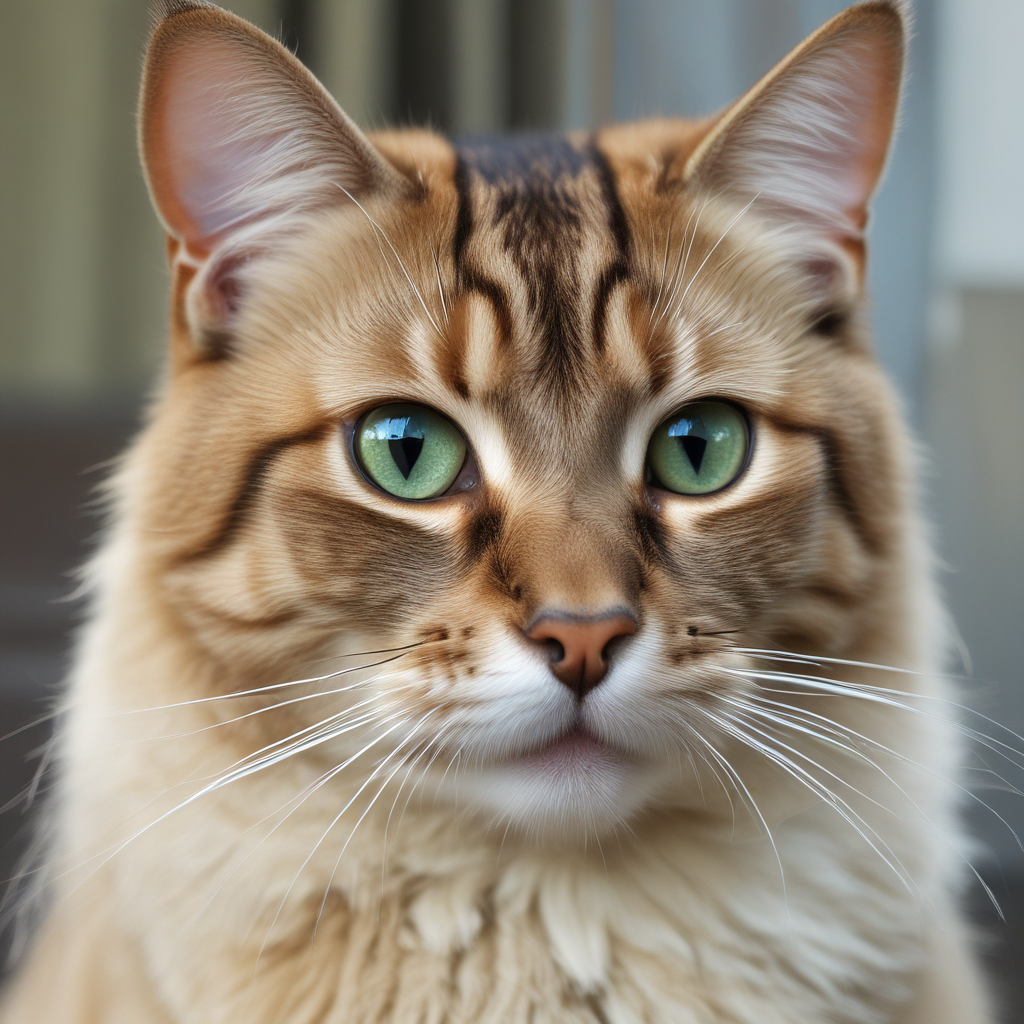} &
      \includegraphics[width=0.33\linewidth]{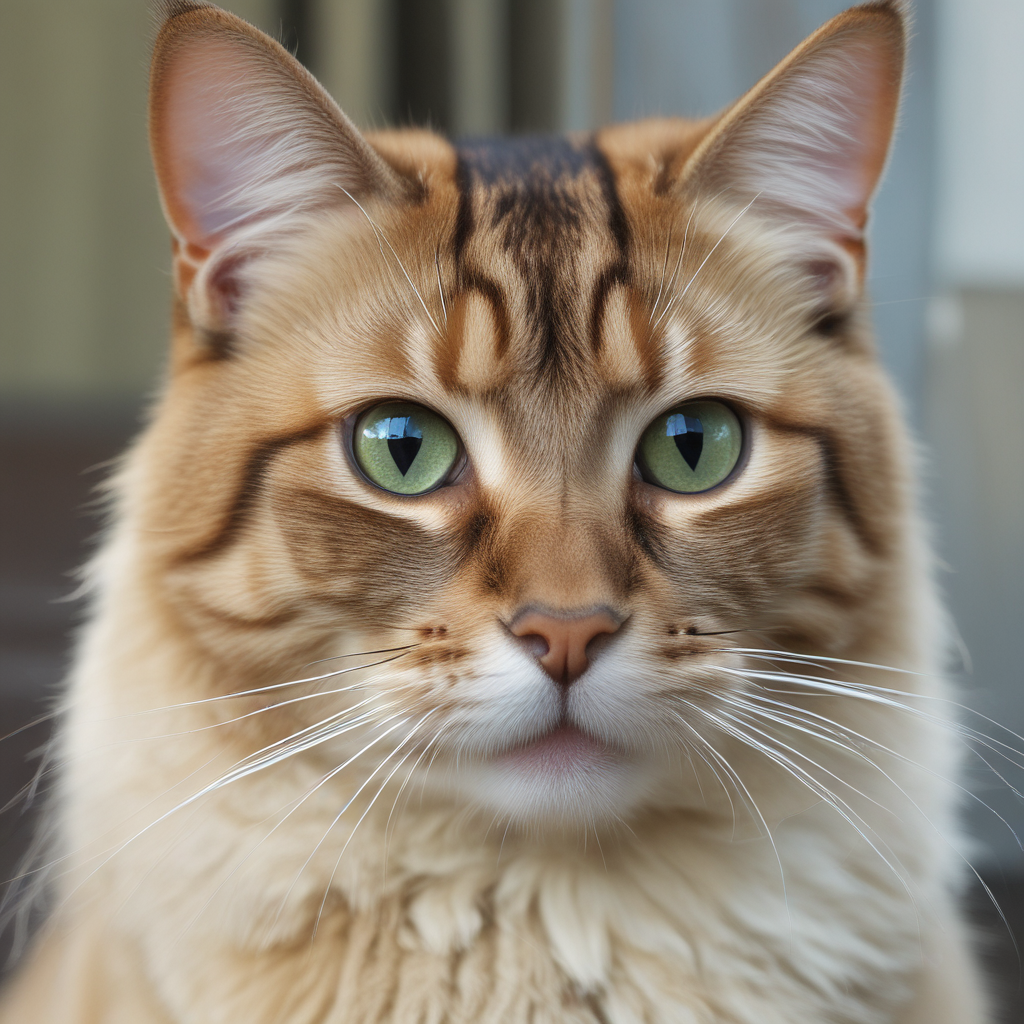} \\
      \includegraphics[width=0.33\linewidth]{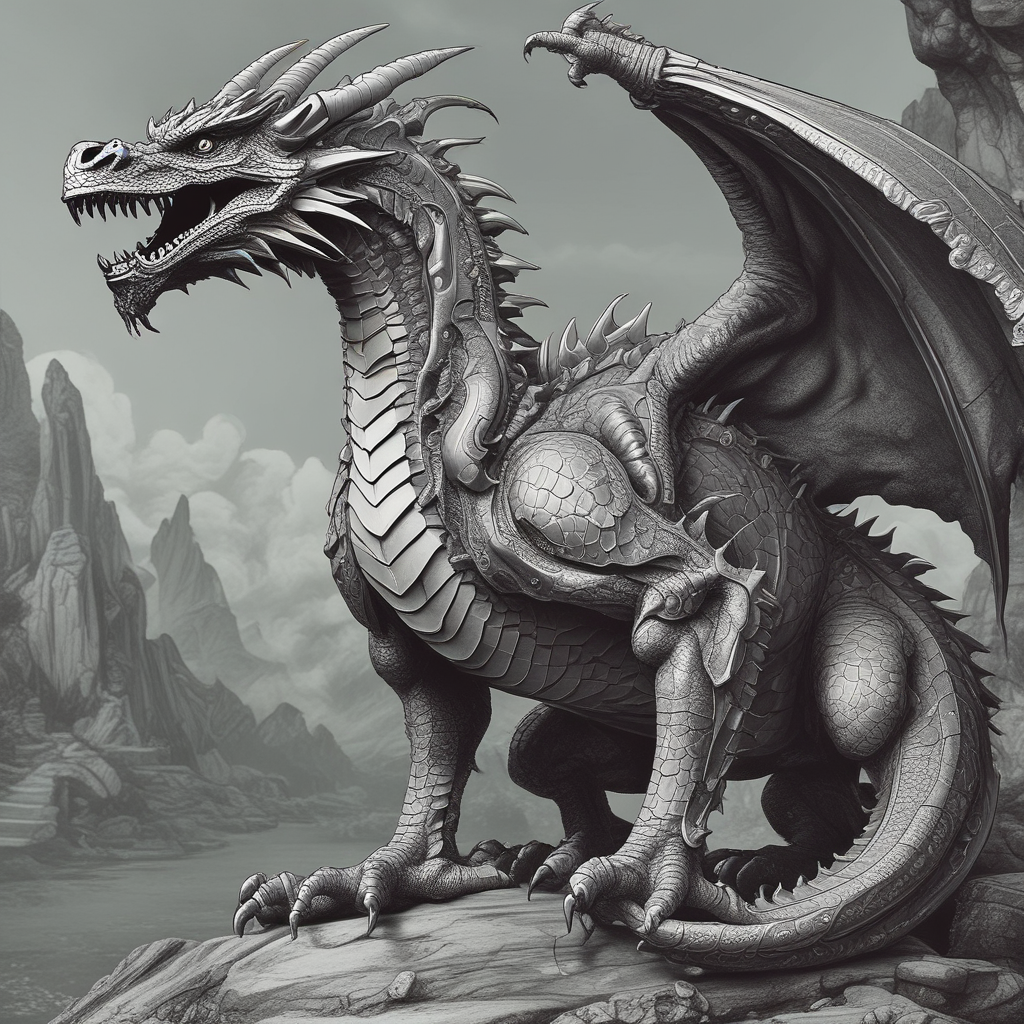} &
      \includegraphics[width=0.33\linewidth]{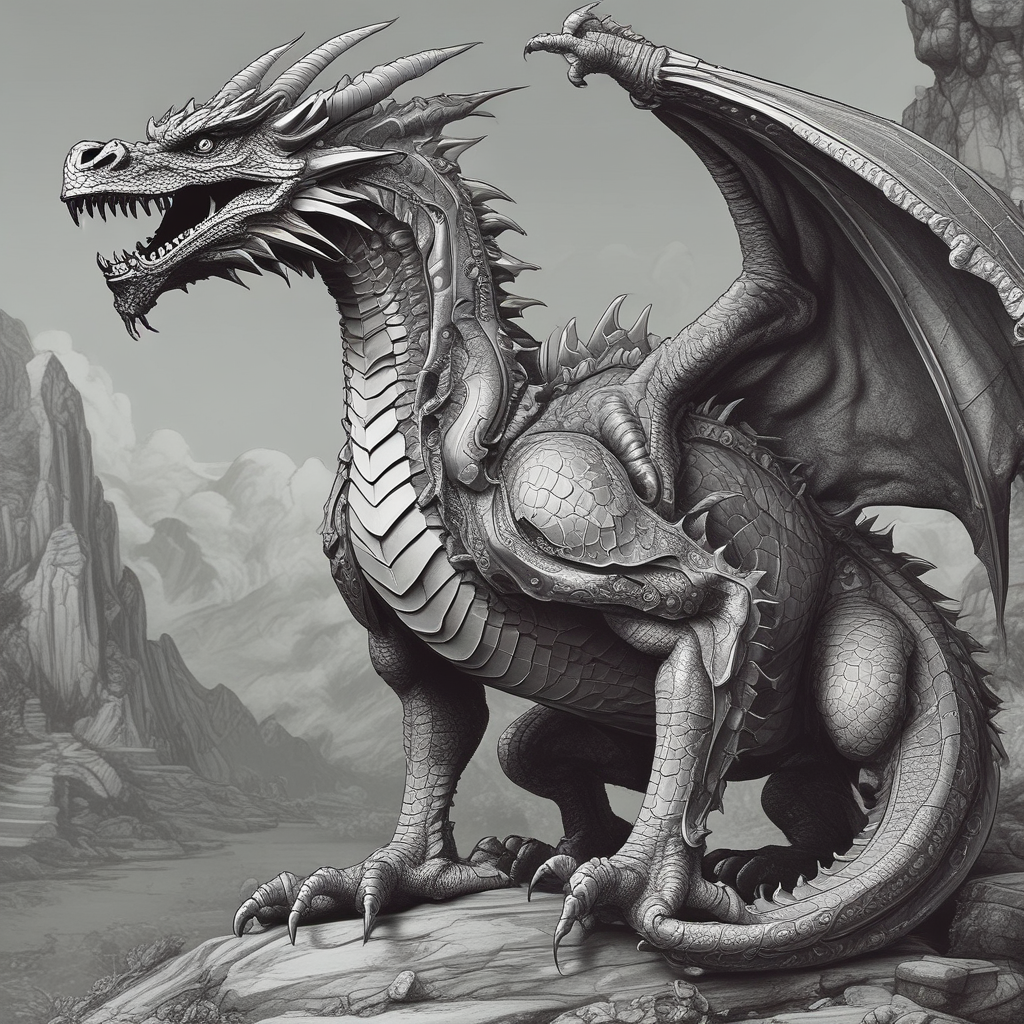} &
      \includegraphics[width=0.33\linewidth]{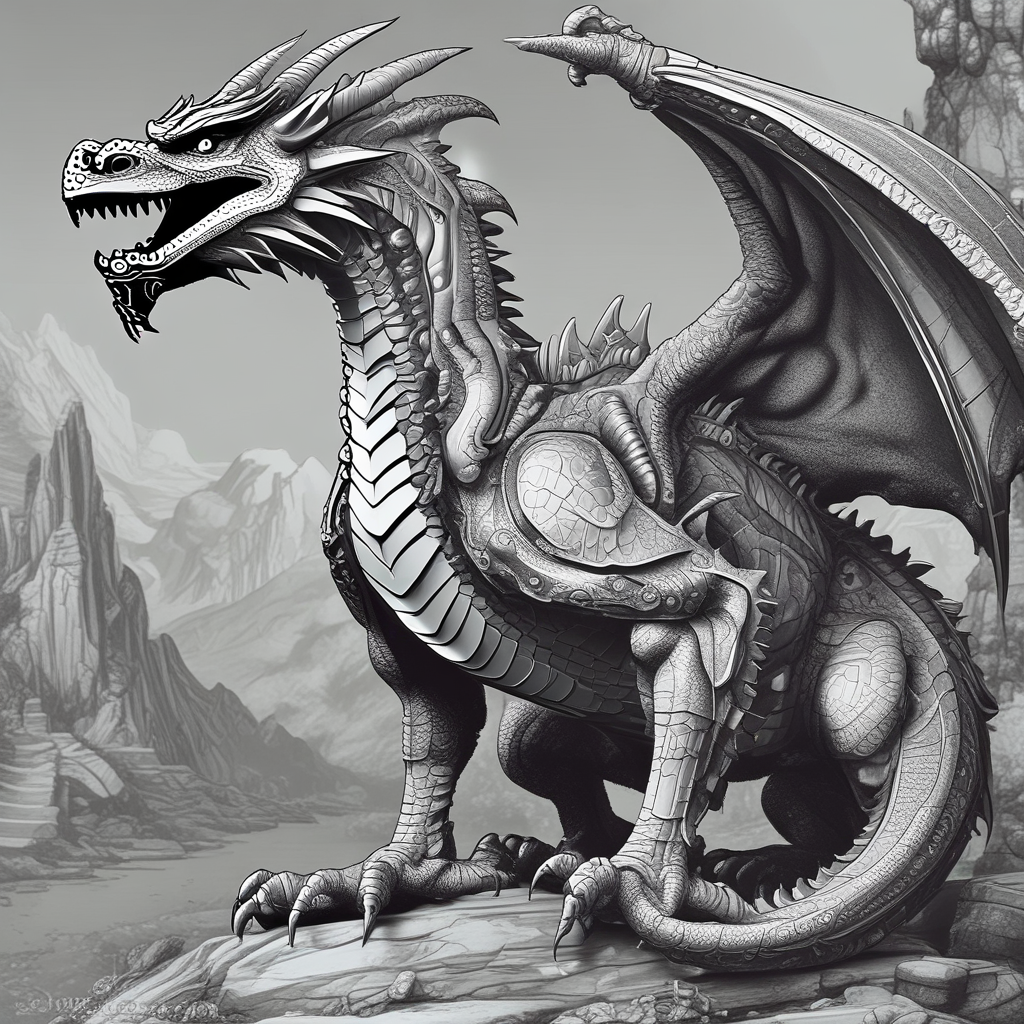} \\
      \end{tabular}%
    }
    \captionsetup{type=figure}
    \caption{Continuous edits along the first principal direction in SDXL-Lightning. Top: prompt ``cat''; bottom: prompt ``dragon''.}
    \label{fig:sdxl_diverse}
  \end{minipage}\hfill
  \begin{minipage}[t]{0.48\columnwidth}
    \vspace{0pt}
    \centering
    \includegraphics[width=\linewidth]{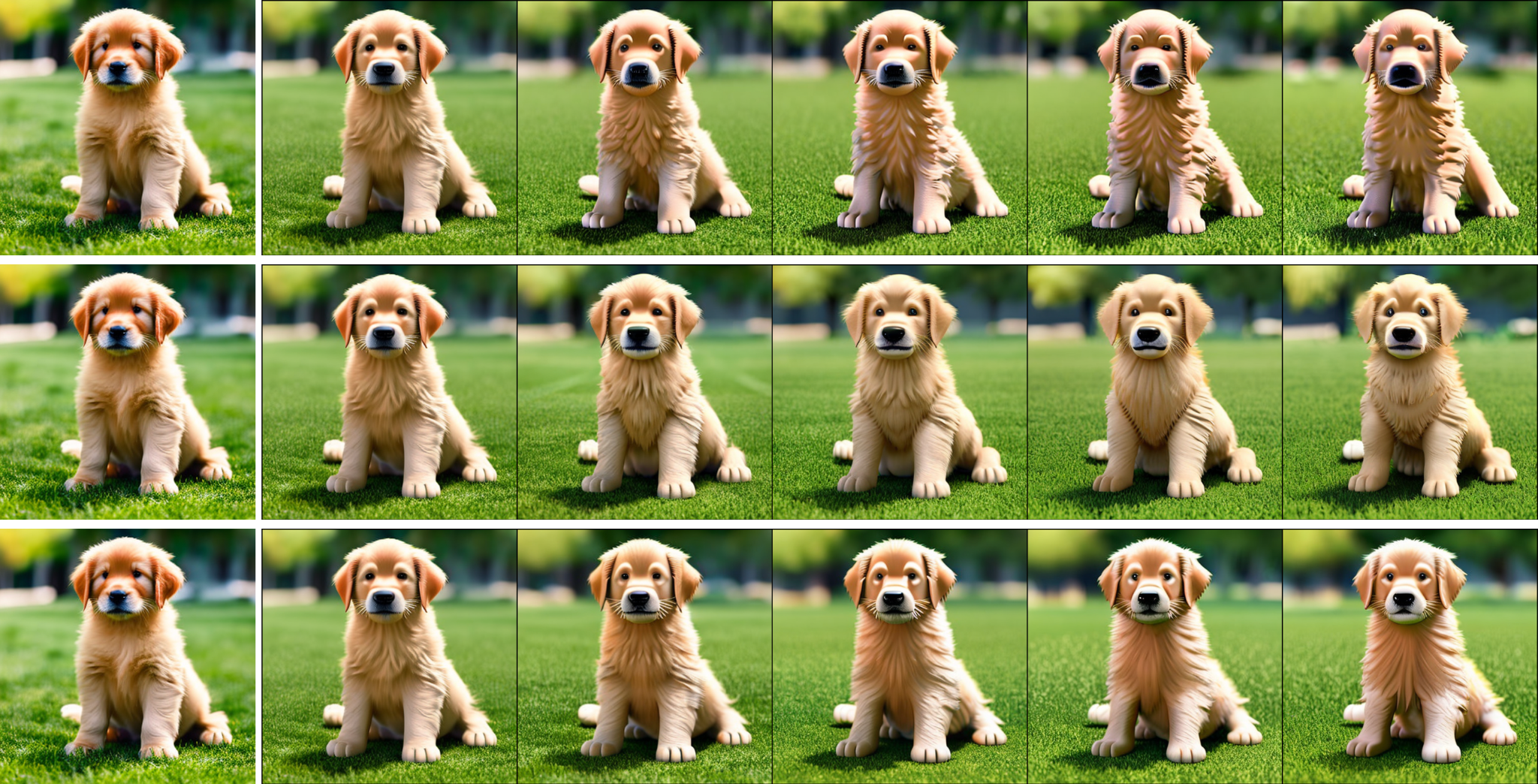}
    \captionsetup{type=figure}
    \caption{Compatibility with edit prompts in \textit{Stable Diffusion}. The leftmost image is generated \emph{without} the edit prompt. The following images exhibit tangent--space edits under the prompt ``cartoon style.''}
    \label{fig:dogs}
  \end{minipage}
  \vspace{-3pt}

\end{figure}

\subsection{Ablation Studies} 
We study the impact of the main design choices in GeoEdit, including tangent projection, perturbation strength, edit step size, and local subspace dimension.

\textbf{Tangent projection under CLIP guidance.} GeoEdit is compared against a no-projection variant that applies raw ambient-space gradients; both use the same few-step noising–denoising schedule. As shown in Fig.~\ref{fig:ablation}, tangent projection preserves the on-manifold component of the update and produces clear, semantic edits with high fidelity, whereas omitting projection pushes updates largely along normal directions that are subsequently damped by refinement, resulting in minimal visual change despite similar step sizes.

\captionsetup[subfigure]{skip=3pt, belowskip=0pt}

\begin{figure}[t]
  \centering

  \begin{minipage}[t]{0.49\columnwidth}
    \vspace{0pt}
    \centering

    \begin{subfigure}{\linewidth}
      \centering
      \includegraphics[width=\linewidth]{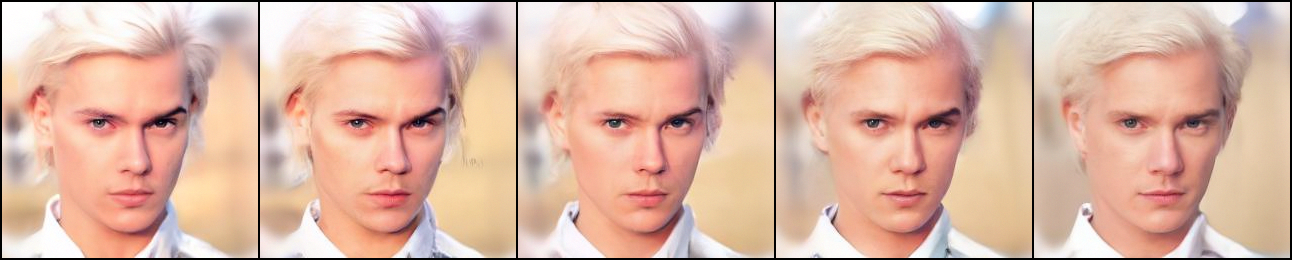}
      \caption{CLIP guided editing without GeoEdit.}
    \end{subfigure}

    \begin{subfigure}{\linewidth}
      \centering
      \includegraphics[width=\linewidth]{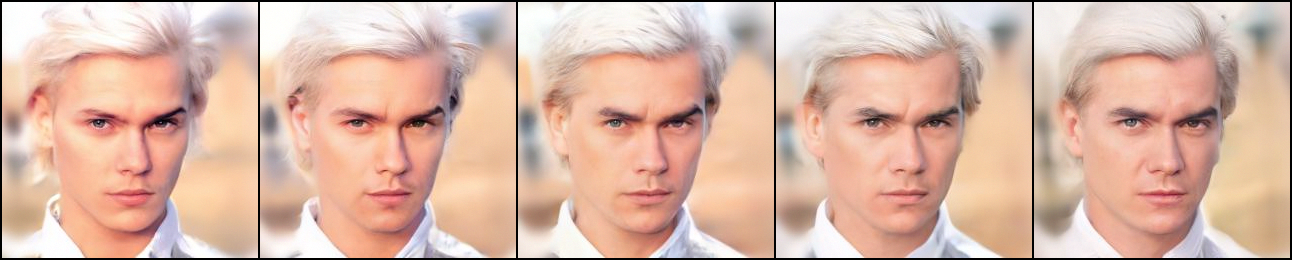}
      \caption{GeoEdit with CLIP guidance.}
    \end{subfigure}

    \captionsetup{type=figure}
    \caption{Ablation study. Both results were produced with CLIP guidance using the identical prompt “an old male.” GeoEdit preserves the on-manifold component and produces edits that are consistent with the given prompt.}
    \label{fig:ablation}
  \end{minipage}%
  \hfill
\begin{minipage}[t]{0.49\columnwidth}
  \vspace{0pt}
  \centering

  \begin{subfigure}{\linewidth}
    \centering
    \includegraphics[width=\linewidth]{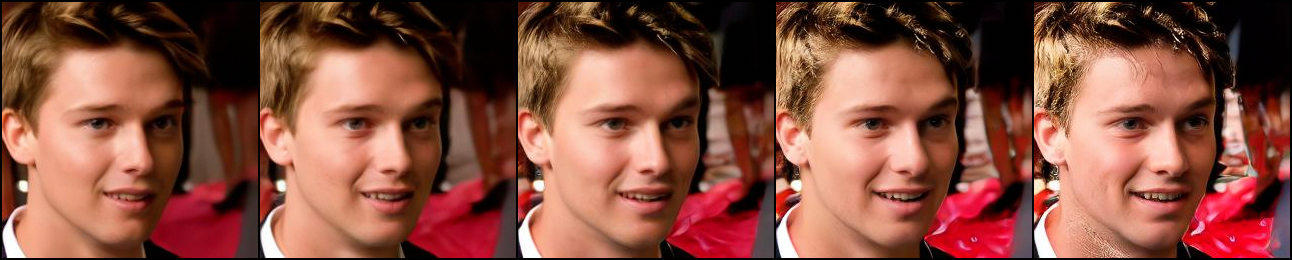}
    \caption{Effect of edit step size $\eta$ in one-step editing.}
    \label{fig:stepsize_ablation}
  \end{subfigure}

  \begin{subfigure}{\linewidth}
    \centering
    \includegraphics[width=\linewidth]{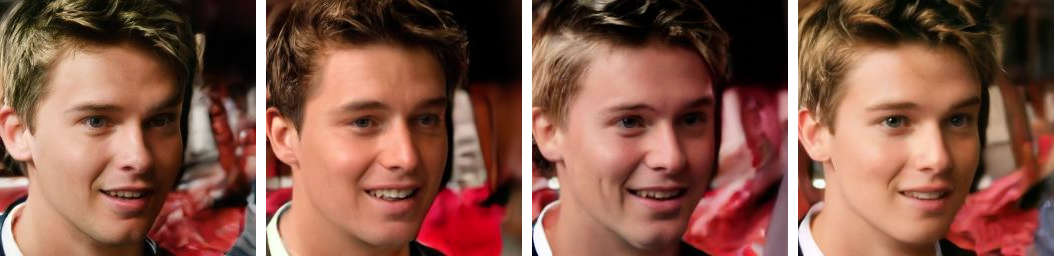}
    \caption{Effect of local subspace dimension $k$.}
    \label{fig:rank_ablation}
  \end{subfigure}

  \captionsetup{type=figure}
  \caption{Additional ablations. (a) Large step sizes eventually exceed the range in which the 5-step retraction can reliably return to the manifold. (b) GeoEdit remains robust across different choices of $k$ when editing along the first principal direction.}
  \label{fig:right_ablation}
  \vspace{-3pt}
\end{minipage}

\end{figure}









\textbf{Edit step size.}
We observe a failure point around $\eta \approx 90$, beyond which the projection step can no longer reliably pull the sample back to the data manifold using a 5-step retraction. 
Fig.~\ref{fig:stepsize_ablation} illustrates this effect with one-step edits at $\eta \in \{0,30,60,90,120\}$.


\textbf{Subspace dimension.}
To study the effect of the local subspace dimension, we fix $\eta=15$ and vary $k \in \{5,10,15,20\}$ while editing along the first principal direction. 
As shown in Fig.~\ref{fig:rank_ablation}, GeoEdit remains visually stable across a broad range of $k$. 
We note that changing $k$ may slightly alter the estimated first principal direction, which can lead to small differences in the resulting edits.


\subsection{Quantitative Evaluation on CelebA-HQ}
We evaluate GeoEdit on CelebA-HQ using both \emph{editing fidelity} and \emph{image quality} following \cite{chen2024exploring}. 
Specifically, we report SSIM and LPIPS to measure content preservation after editing, and FID to quantify the distribution shift relative to the original pretrained DDPM. 
Following the LOCO-Edit protocol, we compute SSIM/LPIPS on 400 images. 

As shown in Table~\ref{tab:celeba_metrics}, GeoEdit matches the best LPIPS while achieving the highest SSIM among all compared methods, indicating stronger structure preservation under comparable perceptual change. 
Table~\ref{tab:fid} further shows that GeoEdit maintains high image quality, with only a small increase in FID compared with the original pretrained generator.

In Fig.~\ref{fig:clip}, We plot the CLIP score along the manifold-optimization edit steps on randomly selected images. It shows a consistent increase as optimization proceeds, indicating that our tangent updates reliably improve alignment with the target text.
\begin{table}[t]
\captionsetup{skip=2pt}
\centering
\small

\begin{subtable}[t]{0.62\columnwidth}
  \centering
  \subcaption{LPIPS and SSIM results}
  \label{tab:celeba_metrics} 
  \setlength{\tabcolsep}{1.8pt}
  \renewcommand{\arraystretch}{1.05}
  \resizebox{\linewidth}{!}{%
  \begin{tabular}{@{}lcccccc@{}}
  \toprule
  \textbf{Metric} &
  \textbf{Pullback} &
  $\boldsymbol{\partial\epsilon/\partial x}$ &
  \textbf{NoiseCLR} &
  \textbf{Asyrp} &
  \textbf{LOCO} &
  \textbf{GeoEdit} \\
  \midrule
  \textbf{LPIPS$\downarrow$} & 0.16 & 0.13 & 0.14 & 0.22 & 0.08 & \textbf{0.08} \\
  \textbf{SSIM$\uparrow$}    & 0.60 & 0.66 & 0.68 & 0.68 & 0.71 & \textbf{0.79} \\
  \bottomrule
  \end{tabular}%
  }
\end{subtable}\hfill
\begin{subtable}[t]{0.35\columnwidth}
  \centering
  \subcaption{FID results}
  \label{tab:fid} 
  \setlength{\tabcolsep}{10pt}
  \renewcommand{\arraystretch}{1.15}
  \resizebox{\linewidth}{!}{%
  \begin{tabular}{lcc}
  \toprule
   & \textbf{Pretrained DDPM} & \textbf{GeoEdit} \\
  \midrule
  \textbf{FID$\downarrow$} & 29.3 & 31.2 \\
  \bottomrule
  \end{tabular}%
  }
\end{subtable}

\caption{Quantitative results on CelebA-HQ.}
\label{tab:celeba_fid}
\vspace{-5pt}
\end{table}



\subsection{Computational Efficiency}
We next compare wall-clock runtime against LOCO-Edit on a single NVIDIA RTX PRO 6000 GPU. For GeoEdit, we explicitly separate the one-time \emph{frame construction} cost (ensemble denoising and PCA) from the \emph{per-edit} cost (tangent update and 5-step retraction). 
For LOCO-Edit, we report the corresponding initialization and edit-time costs under the same setup.

As in Table~\ref{tab:time}, GeoEdit is more efficient in both stages. 
With 10 perturbations, GeoEdit requires 8.38\,s to construct the local frame and only 0.5\,s per edit, while LOCO requires 20.4\,s for initialization and 18.6\,s per edit. 
When increasing the number of perturbations from 10 to 20, GeoEdit's initialization time increases to 18.9\,s and the per-edit time increases to 1.5\,s, exhibiting approximately linear scaling. 
For reference, a single DDPM generation pass takes about 2.1\,s.


\begin{figure*}[t]
\captionsetup{skip=2pt}
\centering

\begin{minipage}[t]{0.48\textwidth}\vspace{0pt}
  \centering
  \includegraphics[width=\linewidth]{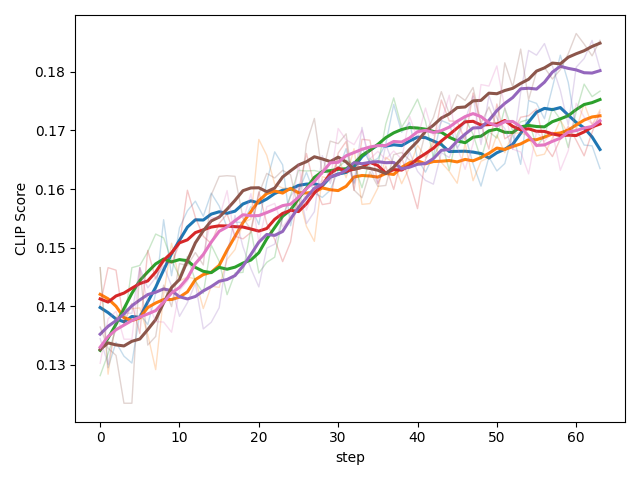}
  \caption{CLIP score (smoothed).}
  \label{fig:clip}
\end{minipage}\hfill
\begin{minipage}[t]{0.49\textwidth}\vspace{0pt}
  \centering
  \small
  \setlength{\tabcolsep}{5pt}
  \renewcommand{\arraystretch}{1.1}

  \captionof{table}{Wall-clock runtime comparison.}
  \label{tab:time}

  \resizebox{\linewidth}{!}{%
  \begin{tabular}{lcc}
    \toprule
    \textbf{Method} &
    \textbf{\shortstack{Init.\\Time (s)$\downarrow$}} &
    \textbf{\shortstack{Per-Edit\\Time (s)$\downarrow$}} \\
    \midrule
    LOCO    & $20.36 \pm 0.09$         & 18.59 \\
    GeoEdit & $\mathbf{8.38 \pm 0.01}$ & \textbf{0.51} \\
    \bottomrule
  \end{tabular}}
\end{minipage}
\vspace{-3pt}
\end{figure*}
\section{Conclusion and Future Work}
We propose a manifold-aware, Jacobian-free diffusion editing framework that estimates local tangents from sample statistics and alternates tangent updates with diffusion projections. The tangent frame enables on-manifold edits and control strength via step count, avoiding re-diffusion and extra training. We provide theoretical and empirical support for the tangent estimator, and experiments show smooth semantic traversals and effective CLIP-guided optimization.

A key future work is the design of well-posed, task-aligned objectives in a latent space, i.e., surrogates that correlate with human perception, remain differentiable and numerically stable. Another promising direction for future work is to extend our framework to optimized shape generation.



\par\vfill\par

\clearpage  


%
%
\bibliographystyle{splncs04}
\bibliography{main}

\clearpage
\section{Appendix}
\label{sec:appendix}

\subsection{Related Works}

\label{sec:related}

\textbf{Manifold-aware Generative Editing.} The manifold hypothesis posits that high-dimensional data concentrate near a lower-dimensional manifold, making intrinsic-dimension estimation central to representation learning \cite{fefferman2016testing}. Locally, this structure is encoded by tangent spaces, which capture the effective degrees of freedom and stitch together to reveal global geometry. Beyond intuition, recent evidence indicates that trained diffusion models encode such manifolds: the learned score aligns with normal directions, exposing manifold geometry \citep{stanczuk2024diffusion}; scalable Fokker–Planck–based estimators recover local structure from samples \citep{kamkari2024geometric}; and large text-to-image models (e.g., Stable Diffusion) exhibit prompt-dependent manifold organization across layers \citep{kvinge2023exploring}. These observations motivate casting editing as guided motion on the intrinsic manifold, rather than in the ambient space where semantic consistency degrades. Prior efforts pursue this principle by constraining samplers to preserve manifold structure \citep{he2023manifold}, refining classifier-free guidance to limit off-manifold drift \citep{chung2024cfg++,kwon2025tcfg}, and enforcing manifold constraints in inverse problems \citep{chung2022improving}. Other lines probe local geometry via implicit manifold-aware local sampling \citep{luzi2022boomerang}, score-decomposed translation \citep{sun2023sddm}, and subspace/Riemannian analyses that reveal controllable directions and curved metric structure \citep{chen2024exploring,park2023understanding}. Practical editing systems leverage guidance for real-image inversion and manipulation \citep{mokady2023null}, trace plausible manifold trajectories over time \citep{rotstein2025pathways}, and mitigate attribute skew with unbiased manifold-aware guidance \citep{su2024unbiased}.

\textbf{Training-free Edit.} 
A line of training-free methods provides continuous knobs to regulate edit strength during inference. Noise-based approaches (e.g., SDEdit~\cite{meng2021sdedit}) vary the starting timestep/perturbation depth, yielding a smooth trade-off between fidelity and change. Guidance-based controls adjust classifier-free guidance or related score-mixing schedules, while attention-based editors (e.g., Prompt-to-Prompt~\cite{hertz2022prompt}, MasaCtrl~\cite{Cao2023MasaCtrl}) interpolate or lock cross/self-attention maps across layers and timesteps to dial the magnitude and locality of edits. Subspace-oriented techniques (e.g., LOCO Edit~\cite{chen2024exploring}) derive low-dimensional semantic directions whose coefficients act as fine-grained strength sliders. Image-to-image morphing frameworks (e.g., DiffMorpher~\cite{zhang2024diffmorpher}, IMPUS~\cite{yang2023impus}) directly produce a continuum of intermediate states by interpolating latents, features, or attention, often exposing a single morph ratio that maps monotonically to perceived strength. Despite their flexibility, many trajectory-steering methods still couple strength with the sampling path so that changing strength commonly triggers re-diffusion.

\subsection{Proofs}

\begin{lemma}[Truncation error for $\Phi$]
Assume $\Phi\in C^3$ in a neighborhood of $z\in\mathbb{R}^n$. Then, for any
$\xi\in\mathbb{R}^n$ and any $\sigma>0$ sufficiently small,
\[
\Phi(z+\sigma\xi)-\Phi(z)
= \sigma\,\mathrm D\Phi(z)\,\xi
\;+\; \tfrac{\sigma^2}{2}\,\mathrm D^2\Phi(z)[\xi,\xi]
\;+\; R,
\]
where $\mathrm D^2\Phi(z)[\cdot,\cdot]$ denotes the bilinear action of the Hessian
and $R$ is a higher-order residual satisfying $\|R\|=O(\sigma^3\|\xi\|^3)$ as $\sigma\to0$.
\end{lemma}

\begin{proof}
By multivariate Taylor’s theorem with (integral) remainder for $C^3$ maps,
for $h=\sigma\xi$ we have
\begin{align*}    
\Phi(z+h)=\Phi(z)+\mathrm D\Phi(z)h+\tfrac12\,\mathrm D^2\Phi(z)[h,h] \\
+\int_0^1 \tfrac{(1-t)^2}{2}\,\mathrm D^3\Phi(z+th)[h,h,h]\;dt.
\end{align*}
Set $h=\sigma\xi$. The integral term defines $R$ and, by continuity and boundedness
of $\mathrm D^3\Phi$ on a neighborhood of $z$, satisfies
$\|R\|\le C\,\sigma^3\|\xi\|^3$ for $\sigma$ small enough.
\end{proof}

\begin{theorem}[Subspace deviation from $k$ ambient secants; vanishing bound as $\rho,\kappa\to0$]
\label{thm:tangent-approx}
Fix $k\le d$ and perturbations $\{\xi_i\}_{i=1}^k$, and write $\Xi=[\xi_1,\ldots,\xi_k]$.
Let $T:=T_{x_\parallel}\mathcal M$, $S_k:=\mathrm{span}\{\Delta_1,\ldots,\Delta_k\}$, and
$s_{\min}:=\sigma_{\min}(D\Pi(z)\,\Xi)>0$.
Assume in addition that $
\|P_T^\perp D\Phi(z)\|_2 \;\le\; L_\perp\,\rho$,
and the normal second–order deviation is curvature–controlled:
$\|P_T^\perp R\|_2\le C_{\mathrm{curv}}\kappa\,\sigma^2\|\Xi\|_2^2$,
where $R:=[R_1,\ldots,R_k]$ with $R_i:=\tfrac{\sigma^2}{2}\,D^2\Phi(z)[\xi_i,\xi_i]$.
Then, for sufficiently small $\sigma,\rho$,
\[
\|P_T^\perp P_{S_k}\|_2
\;\le\;
\frac{L_\perp\,\rho\,\|\Xi\|_2}{s_{\min}}
\;+\;
\frac{C_{\mathrm{curv}}\,\kappa\,\sigma\,\|\Xi\|_2^2}{s_{\min}}
\;+\;
\frac{C_3\,\sigma^2\,\|\Xi\|_2^3}{s_{\min}}.
\]
In particular, the first two terms vanish as $\rho,\kappa\to0$ for fixed small $\sigma$.
\end{theorem}

\begin{proof}
Let $x=\Phi(z)\in\mathcal T_\rho$, $x_\parallel=\pi(x)$, and let $P_T$ be the orthogonal projector onto
$T:=T_{x_\parallel}\mathcal M$. For $i=1,\ldots,k$, set $z_i:=z+\sigma\xi_i$ and define the one-sided
secants $\Delta_i:=\Phi(z_i)-\Phi(z)$. Stack them as $Y:=[\Delta_1,\ldots,\Delta_k]\in\mathbb R^{n\times k}$ and write
$S_k:=\mathrm{range}(Y)$.

By the $C^3$ regularity of $\Phi$, for each $i$,
\begin{align*}
\label{eq:taylor}
\Delta_i
&= \sigma\,D\Phi(z)\,\xi_i \;+\; \tfrac{\sigma^2}{2}\,D^2\Phi(z)[\xi_i,\xi_i] \;+\; r_i \\
&= \sigma\,D\Phi(z)\,\xi_i \;+\; \widetilde R_i,
\end{align*}
where we set $\widetilde R_i:=\tfrac{\sigma^2}{2}\,D^2\Phi(z)[\xi_i,\xi_i]+r_i$ and $\|r_i\|\le C_3'\,\sigma^3\|\xi_i\|^3$ for some constant $C_3'>0$.

Since $
Y=\big[\Delta_1,\ldots,\Delta_k\big]
=\sigma\,D\Phi(z)\,\Xi\;+\;R\;+\;r,
$, projecting by $P_T^\perp$ and using the near-tangency and curvature-controlled bounds yields
\begin{equation}\label{eq:PTperpY}
\|P_T^\perp Y\|_2
\ \le\ \sigma L_\perp\rho\,\|\Xi\|_2
\;+\; C_{\mathrm{curv}}\kappa\,\sigma^2\|\Xi\|_2^2
\;+\; C_3\,\sigma^3\|\Xi\|_2^3,
\end{equation}
where we used $\|\Xi\|_F\le \sqrt{k}\|\Xi\|_2$ and absorbed $\sqrt{k}$ into $C_3$.

Projecting by $P_T$ gives
\[
P_T\Delta_i
= \sigma\,P_T D\Phi(z)\,\xi_i \;+\; P_T R_i \;+\; P_T r_i.
\]
Since $\Pi=\pi\circ\Phi$ and $E_\pi(x) := D\pi(x)-P_T$ with $\|E_\pi(x)\|_2=O(\kappa\rho)$ on $\mathcal T_\rho$, we have
\[
P_T D\Phi(z)=D\pi(x)\,D\Phi(z)\;+\;O(\kappa\rho\,\|D\Phi(z)\|_2).
\]
Writing $A:=D\Pi(z)\,\Xi$ and defining
\[
E_T \;:=\; P_T R \;+\; \big(P_T - D\pi(x)\big)\,\sigma D\Phi(z)\,\Xi \;+\; P_T r,
\]
we have
\begin{align*}\label{eq:PTYdecomp}
P_TY \;&=\; \sigma\,A \;+\; E_T, \\
\|E_T\|_2 \;&\le\; C_T\,\sigma^2\|\Xi\|_2^2 \;+\; C_{T,\mathrm{geom}}\,\kappa\rho\,\sigma\,\|\Xi\|_2 \;+\; C_3''\,\sigma^3\|\Xi\|_2^3,
\end{align*}
for some constants $C_T,C_{T,\mathrm{geom}},C_3''>0$.
 By assumption $s_{\min}:=\sigma_{\min}(A)>0$. Choose $\sigma,\rho$ small so that
$\|E_T\|_2 \ \le\ \tfrac12\,\sigma s_{\min}$.
Then $\sigma_{\min}(P_TY)\ \ge\ \sigma s_{\min}-\|E_T\|_2 \ \ge\ \tfrac12\,\sigma s_{\min}$.

Let $P_{S_k}$ be the orthogonal projector onto $S_k=\mathrm{range}(Y)$. By the simple property of operator norm,
\begin{equation}\label{eq:wedin}
\|P_T^\perp P_{S_k}\|_2
=\sup_{\alpha\neq 0}\frac{\|P_T^\perp Y\alpha\|}{\|P_TY\alpha\|}
\;\le\; \frac{\|P_T^\perp Y\|_2}{\sigma_{\min}(P_TY)}.
\end{equation}
Combining previous inequalities and absorbing absolute constants by renaming $C_3:=2\max\{C_3',C_3''\}$, we get
\[
\|P_T^\perp P_{S_k}\|_2
\ \le\
\frac{L_\perp\rho\,\|\Xi\|_2}{s_{\min}}
\;+\;
\frac{C_{\mathrm{curv}}\kappa\,\sigma\,\|\Xi\|_2^2}{s_{\min}}
\;+\;
\frac{C_3\,\sigma^2\|\Xi\|_2^3}{s_{\min}}.
\]
This proves the stated inequality. In particular, for fixed small $\sigma$, the first two terms vanish as $\rho,\kappa\to 0$; if additionally $\sigma\to 0$, the last term also vanishes, and hence the right-hand side tends to $0$.

\end{proof}

\end{document}